\documentclass[twocolumn]{article} 
\usepackage[a4paper, left=2cm, right=2cm, top=2.5cm, bottom=2.5cm]{geometry}
                        
\usepackage{orcidlink}
\usepackage[american]{babel}
\usepackage{authblk}

\usepackage{natbib} 
\usepackage{mathtools} 
\usepackage{amssymb}
\usepackage{amsthm}
\usepackage{graphicx}
\usepackage{subcaption}
\usepackage{array}
\usepackage{colortbl}
\newcolumntype{C}[1]{>{\centering\arraybackslash}p{#1}}
\usepackage{booktabs} 
\usepackage{tikz} 
\usepackage{makecell}
\usepackage{multirow}
\usepackage[ruled,vlined]{algorithm2e}
\SetKwComment{Comment}{\# }{}
\usepackage[table]{xcolor}
\usepackage{pifont}
\usepackage{rotating}
\usepackage{booktabs}
\usepackage{fontawesome5}
\usepackage{hyperref}

\newcommand{\unblocked}{\textcolor{green}{\ding{51}}} 
\newcommand{\blocked}{\textcolor{red}{\ding{55}}}
\newcommand{\potblocked}{\blocked\hspace{0.1em} or \unblocked}
\newcommand{\EE}{\mathbb{E}}
\newcommand{\RR}{\mathbb{R}}
\newcommand{\Var}{\mathrm{Var}}

\newcommand{\doop}{\mathrm{do}}
\newcommand{\ind}{\perp\!\!\!\perp}

\newcommand{\ipath}{\dashrightarrow}
\newcommand{\pathi}{\dashleftarrow}

\newcommand{\ancestors}{\mathrm{an}}
\newcommand{\argmin}{\mathrm{arg}\,\mathrm{min}}

\newtheorem{theorem}{Theorem}[section]
\newtheorem{lemma}[theorem]{Lemma}
\newtheorem{proposition}[theorem]{Proposition}

\newtheorem{definition}[theorem]{Definition}
\newtheorem{example}[theorem]{Example}

\theoremstyle{remark}
\newtheorem{remark}[theorem]{Remark}

\title{cc-Shapley: Measuring Multivariate Feature Importance Needs Causal Context}
\author[1]{Jörg Martin\thanks{Corresponding author: joerg.martin@ptb.de}}
\author[1,2,3]{Stefan Haufe}

\affil[1]{Physikalisch-Technische Bundesanstalt, Berlin, Germany}
\affil[2]{Technische Universität Berlin, Berlin, Germany}
\affil[3]{Charité – Universitätsmedizin Berlin, Berlin, Germany}

\begin{document}
\maketitle

\begin{abstract}
Explainable artificial intelligence promises to yield insights into relevant features, thereby enabling humans to examine and scrutinize machine learning models or even facilitating scientific discovery. Considering the widespread technique of Shapley values, we find that purely data-driven operationalization of multivariate feature importance is unsuitable for such purposes. Even for simple problems with two features, spurious associations due to collider bias and suppression arise from considering one feature only in the observational context of the other, which can lead to misinterpretations. Causal knowledge about the data-generating process is required to identify and correct such misleading feature attributions. We propose cc-Shapley (causal context Shapley), an interventional modification of conventional observational Shapley values leveraging knowledge of the data's causal structure, thereby analyzing the relevance of a feature in the causal context of the remaining features. We show theoretically that this eradicates spurious association induced by collider bias. We compare the behavior of Shapley and cc-Shapley values on various, synthetic, and real-world datasets. We observe nullification or reversal of associations compared to univariate feature importance when moving from observational to cc-Shapley.
\end{abstract}
\section{Introduction}
\label{sec:introduction}
Feature attribution is a common paradigm in the field of Explainable Artificial Intelligence (XAI). 
Its goal is to quantify the importance of each feature by attributing a score to it.
It is often assumed that using such techniques allows us to check whether a certain model is ``correct'', i.e., whether it misbehaves on certain data points or whether it uses unexpected, possibly unfavorable, structures for its prediction \citep{ribeiro2016should,lapuschkin2019unmasking,anders2022finding}. Likewise, XAI is hoped to boost scientific discovery by unveiling hidden multivariate patterns in the data that are associated with prediction targets of interest \citep{jimenez2020drug,wong2024discovery,samek2019towards}. All of these promises are based on the assumption that the attributions gained with such methods do not show spurious patterns and can be interpreted unambiguously. This work aims to show that this assumption is an illusion unless we include causal knowledge into our considerations.

Our analysis focuses on Shapley values, a concept of game-theoretical origin that has become a well-established approach to XAI. When predicting a target variable $Y$ from features $\mathcal{F}=\{X_1,\ldots,X_n\}$ the \emph{Shapley value} of a feature $X_j$ is defined as\footnote{There are various versions of \eqref{eq:shapley_values}, cf. Appendix \ref{app:on_the_used_version_of_shapley_values}.} \citep{shapley1953value}
\begin{align}
    \label{eq:shapley_values}
    \phi(X_j) = \sum_{\mathcal{S}\subseteq \mathcal{F}\backslash\{X_j\}} {\small \frac{|\mathcal{S}|! (|\mathcal{F}|-|\mathcal{S}|-1)!}{|\mathcal{F}|!}} I_{\mathcal{S}}(X_j)\,,
\end{align}
where we introduce the notation 
\begin{align}
    \label{eq:obs_importance}
    I_{\mathcal{S}}(X_j) = \EE[Y|X_j, \mathcal{S}] - \EE[Y|\mathcal{S}]
\end{align}
for the change in prediction when we add observation $X_j$ to the observed features $\mathcal{S}\subseteq \mathcal{F}$. 
Formula \eqref{eq:shapley_values} computes the weighted sum of these changes, where the combinatorial weights can be interpreted through the probability of randomly drawing $\mathcal{S}$, cf. Appendix \ref{app:on_the_used_version_of_shapley_values}. For binary $Y\in\{0,1\}$, as considered below, \eqref{eq:obs_importance} simplifies to a difference of probabilities.

The computational complexity behind \eqref{eq:shapley_values} quickly increases with $|\mathcal{F}|$. For this reason, several scalable modifications have been proposed, e.g., by \cite{lundberg2017unified} and \cite{janzing2020feature}. We show, however, that even without any approximations put on top, $\phi$ is unsuitable for the purposes generally targeted by XAI. To illustrate this, we will use the following easy running example throughout this work.

\begin{example}[Breakfast and diabetes]
\label{ex:diabetes_breakfast}
Suppose blood glucose $G$ is measured in a patient to assess whether he or she has diabetes ($Y=1$) or not ($Y=0$). Unfortunately, the patient ignores the doctor's request not to eat breakfast before the test. We assume that the measured value $G$ follows the (simplified) relation
\begin{align}
    \label{eq:breakfast_diabetes}
    G = 85 + 0.4 \cdot C + 40 \cdot Y + U \,,
\end{align}
where $C$ denotes the carbohydrate intake during breakfast and $U$ denotes independent noise. We assume that $Y\sim \mathrm{Bernoulli}(0.15)$, $C\sim \mathcal{N}(60;25^2)$ and $U\sim \mathcal{N}(0;10^2)$.
\end{example}

The graph in Figure \ref{fig:causal_graph_diabetes_breakfast} illustrates the causal structure of Example \ref{ex:diabetes_breakfast}: the value of $G$ is caused by the variables $Y$ and $C$. The plot on the left of Figure \ref{fig:shapley_values_diabetes_breakfast} shows the Shapley values \eqref{eq:shapley_values} for data generated from this simple problem. A red color indicates a high value of a feature, whereas a blue color indicates a low value. A positive Shapley value $\phi(G)=\frac{1}{2} I_{\emptyset}(G)+\frac{1}{2} I_{C}(G)$ indicates a positive association whereas larger $G$ tend to co-occur with larger $Y$ if $G$ is either considered alone ($\mathcal{S}=\emptyset$) or in the context of a fixed $C$ ($\mathcal{S}=\{C\}$). As red points in Figure \ref{fig:shapley_values_diabetes_breakfast} are concentrated on the right for $G$, we could conclude that high $G$ values makes diabetes more likely, which matches intuition. We call this \emph{positive relevance} in the following.

The Shapley values for $C$, on the other hand, indicate a \emph{negative relevance}: high values of $C$ are associated with a lower diabetes probability. Does this imply that a single high carbohydrate intake makes it less likely that the patient has diabetes? This seems absurd and rightly so. The underlying reason is instead something which one might dub the ``explain away effect'' but which is typically known as \emph{suppression} in the literature \citep{conger1974revised}: When the patient consumed many carbohydrates there is no need for diabetes to explain high values of $G$.
The presence of suppression in our example can be read off the causal graph in Figure \ref{fig:causal_graph_diabetes_breakfast}: the suppressor $C$ is connected to the target $Y$ via a node $G$ with two incoming arrows (known as a \emph{collider}). Conditioning on a collider leads to a spurious association, known as \emph{collider bias}, which makes $C$ a suppressor of $G$. Collider bias and suppression are in this work considered as two sides of the same coin, cf. Section \ref{subsec:collider_bias}.
 Not every feature attributed negative relevance by $\phi$ will automatically be a suppressor. In other examples a feature can, of course, be of negative relevance as it really causes the target to decrease. Apparently, the ``right'' interpretation of a chart like the left one in Figure \ref{fig:shapley_values_diabetes_breakfast} can vary from situation to situation. This ambiguity can easily lead to wrong conclusions and undermines the usability of XAI for purposes such as model analysis, let alone for scientific discovery.
 
\begin{figure}[t]
    \centering
    \begin{subfigure}[t]{0.5\textwidth}
  \centering
    \begin{tikzpicture}[scale=0.6,every node/.style={circle, draw}, node distance=2cm]
      \node (G) {G};
      \node (C) [left of=G] {C};
      \node (Y) [right of=G] {Y};
      \draw[->] (Y) -- (G);
      \draw[->] (C) -- (G);
    \end{tikzpicture}
    \caption{Causal graph for Example \ref{ex:diabetes_breakfast}}
    \label{fig:causal_graph_diabetes_breakfast}
    \end{subfigure}
    \begin{subfigure}[t]{0.5\textwidth}
    \includegraphics[width=1\linewidth]{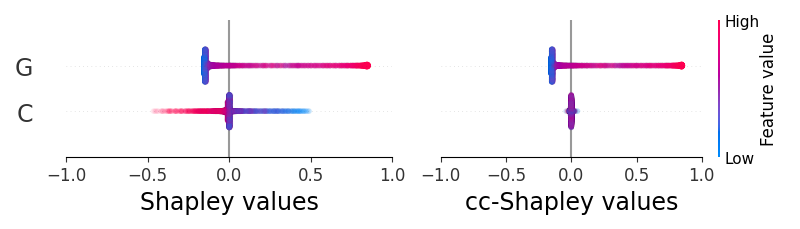}
    \caption{Shapley (left) and cc-Shapley values (right) for Ex. \ref{ex:diabetes_breakfast}}
    \label{fig:shapley_values_diabetes_breakfast}
    \end{subfigure}
    \caption{Causal graph for Example \ref{ex:diabetes_breakfast} (top), together with the according results for conventional Shapley values \eqref{eq:shapley_values} (bottom - left) and the cc-Shapley values \eqref{eq:shapley_values_do} (bottom - right). Experimental details are in Appendix \ref{app:details_on_diabetes_breakfast}.
    } 
    \label{fig:diabetes_breakfast_summary}
\end{figure}
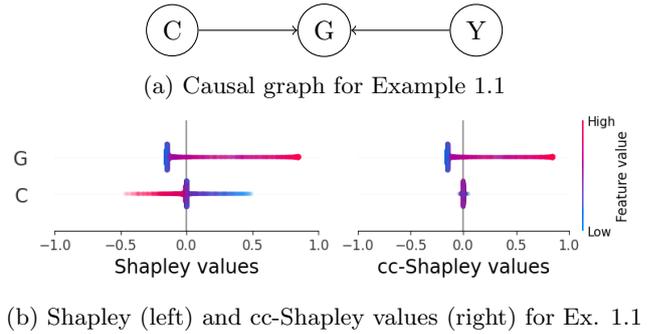

 Why were we able to tell that a suppression effect rather than a potential health benefit of one carbohydrate-rich breakfast is a more likely interpretation for the negative relevance of $C$ in Figure \ref{fig:shapley_values_diabetes_breakfast}? We had to use something that a purely data-driven model does not have: an understanding of the real world, encoded in the causal diagram in Figure \ref{fig:causal_graph_diabetes_breakfast}. In this article we propose to use this understanding for our treatment of the context variables $\mathcal{S}$ in \eqref{eq:shapley_values} via techniques from \emph{causal inference} \citep{pearl2009causality}. This causal modification of \eqref{eq:shapley_values} is called \emph{cc-Shapley} values, and is introduced in Definition \ref{def:cc_Shapley}. For Example \ref{ex:diabetes_breakfast}, we depict the cc-Shapley values on the right side of Figure \ref{fig:shapley_values_diabetes_breakfast}. We observe that the cc-Shapley values do not attribute $C$ any importance for diabetes risk, which matches the intuition.

The main contributions of this article are as follows:
\begin{itemize}
    \item We highlight a fundamental problem underlying non-causal XAI methods such as \eqref{eq:shapley_values}. In causal inference, this problem is known as collider bias or suppression.
    \item We propose cc-Shapley values, which modify \eqref{eq:shapley_values} to use causal information. This is, to the best of our knowledge, the first approach designed to avoid collider bias without the need to restrict oneself to univariate feature importance.
    \item We present theoretical and experimental results regarding the computability and behavior of cc-Shapley values in the presence of collider bias. We use several synthetic and one real world example.
\end{itemize}

\paragraph{Comparison with the literature.} There are a plethora of XAI methods in the literature, cf. for example the reviews by \cite{angelov2021explainable,minh2022explainable} and \cite{mersha2024explainable}. While this work focuses entirely on Shapley values, the criticism raised here goes beyond this specific method. Typically, non-causal methods use only observational ingredients: the data and/or the model (which is also trained from the data). As causal information is not recoverable from observations alone in general and as collider bias is a causal phenomenon, it is implausible that any combination of purely observational information can eradicate the influence of collider bias.
Indeed, \cite{wilming2023theoretical} demonstrate the susceptibility of many popular XAI methods to suppression.

There is a growing awareness that the current approach to XAI needs scrutiny or even revision
\citep{freiesleben2023dear,haufe2024explainable,wilming2022scrutinizing} with some works discussing the need and possibility for combining XAI with causal concepts, e.g., \citep{carloni2025role,beckers2022causal,karimi2025position,holzinger2019causability,scholkopf2022causality}.
Various works propose using the underlying causal structure to produce counterfactual explanations, see, e.g., the works by \cite{karimi2021algorithmic,karimi2020algorithmic,konig2023improvement}. The focus in these works is mostly to produce counterfactuals that are coherent with the actual world, but the topic of suppression is not discussed in this context.

\cite{janzing2020feature,janzing2024quantifying,heskes2020causal,frye2020asymmetric} and \cite{jung2022measuring} also discuss variants of Shapley values with causal ingredients. We discuss the relation of our method to these variants in Section \ref{subsec:using_the_causal_context} and, in greater detail, in Appendix \ref{app:detailed_comparision_with_other_methods}.
\cite{wilming2023theoretical} study (conventional) Shapley values for a simple suppressor problem for which no relevance is attributed to a suppressor; this statement is however only true for their considered example and a specific approximation to Eq.~\eqref{eq:shapley_values}. \cite{haufe2014interpretation} and \cite{clark2025correcting} discuss the problem of suppression and observe that a univariate approach to feature importance avoids these problems. We will discuss the relation of our work to this idea in Section \ref{subsec:collider_bias} including a discussion why univariate importance is in general not sufficient to assess the relevance of a feature.

\section{Background}
\label{sec:background}

\subsection{Basics of causal inference}
\label{subsec:basics_of_causal_inference}
We here recall core concepts of causal inference needed in this article and refer to \cite{pearl2009causality}. In causal inference, the data genererating process behind a set of variables is described by a \emph{structural causal model (SCM)}, which is a triplet $\mathcal{M}=(\mathcal{V},(f_X)_{X\in \mathcal{V}},(U_X)_{X\in \mathcal{V}})$, where
\begin{itemize}
    \item $\mathcal{V}$ is the set of variables,
    \item $(U_X)_{X\in \mathcal{V}}$ is a set of independent\footnote{Throughout this work, we assume that the SCM satisfies the local Markov assumption, i.e., there are no hidden confounders outside the set $\mathcal{V}$.} random variables, representing noise, and
    \item each variable $X\in \mathcal{V}$ is related to its \emph{parents} $\mathrm{pa}_{X}\subseteq \mathcal{V}\backslash\{X\}$ and its noise variable $U_X$ via an \emph{assignment function} $X=f_X(\mathrm{pa}_X,U_X)$.
\end{itemize}
To avoid inconsistencies, the relations are assumed to be acyclic: That is, when drawing arrows from $\mathrm{pa}_X$ to $X$ for each $X\in \mathcal{V}$, we arrive at a directed acyclic graph $\mathcal{G}=(\mathcal{V},\mathcal{E})$ with edges $\mathcal{E}$. Given an instance of the noise variables $(U_X)_{X\in \mathcal{V}}$, we can find the values of all $X\in \mathcal{V}$ through evaluation of the assignment functions, starting from the source nodes without parents. The graph $\mathcal{G}$ is called the \emph{causal graph} and represents the causal functional dependencies underlying the data generating process. Figure \ref{fig:causal_graph_diabetes_breakfast} depicts the causal graph for Example \ref{ex:diabetes_breakfast}.

Statistical association between variables can be assessed by studying the paths in $\mathcal{G}$. A \emph{path} between two variables is a concatenation of adjacent edges from $\mathcal{E}$ (independent of the direction of the edges) leading from one variable to the other. If all edges in a path have the same direction, we denote it by $\ipath$ or $\pathi$ and call it a \emph{causal path}. If a causal path consists of only one edge we call it \emph{direct}, otherwise \emph{indirect}.
A path between two sets $\mathcal{S}$ and $\mathcal{S}'$ is called a \emph{backdoor path} if it leaves $\mathcal{S}$ with an incoming edge.
A path is called \emph{blocked}, conditional on a (possibly empty) set $\mathcal{Z} \subseteq \mathcal{V}$, if it 

\begin{itemize}
    \item contains a $Z\in \mathcal{Z}$ that acts on this path as a \emph{chain} $\rightarrow Z \rightarrow $ ( or $\leftarrow Z \leftarrow$) or as a \emph{fork} $\leftarrow Z \rightarrow$, or 
    \item if it contains a \emph{collider} $\rightarrow C \leftarrow $ which is neither contained in $\mathcal{Z}$ nor in the \emph{ancestors} $\ancestors(\mathcal{Z})$ of $\mathcal{Z}$ (the set of nodes that have a causal path leading to $\mathcal{Z}$).
\end{itemize}

A path is called \emph{unblocked} if it is not blocked. If all paths between two sets of variables $\mathcal{S}, \mathcal{S}' \subseteq\mathcal{V}$ are blocked conditional on $\mathcal{Z}$, we call them \emph{d-separated} and write $\mathcal{S} \ind_{\mathcal{G}} \mathcal{S}' | \mathcal{Z}$. d-separation implies independence: $\mathcal{S} \ind_{\mathcal{G}} \mathcal{S}' |\mathcal{Z} \Rightarrow \mathcal{S} \ind \mathcal{S}' | \mathcal{Z} $. The graph $\mathcal{G}$ is called \emph{faithful} if the reverse direction is always true. 

\emph{Intervention} is the process of creating a new SCM from $\mathcal{M}$ by setting a variable $X\in \mathcal{V}$ to a specific value $x$, replacing $f_X(\mathrm{pa}_X,U_X)$ by $\tilde{f}_X=x$. 
Intervention simulates the effect of a variable when we cut off its usual causes and corresponds to deleting all incoming arrows to $X$ in $\mathcal{G}$. We write $\doop(X=x)$ or simply $\doop(X)$ for this operation and $\mathcal{M}^{\doop(X)}$ for the modified SCM. Similarly, we can define an intervention $\doop(\mathcal{X})$ on a subset of variables $\mathcal{X}\subseteq \mathcal{V}$ or a stochastic intervention $\doop(\mathcal{X}\sim q)$ when drawing the values $\mathcal{X}$ from a distribution $q$ instead of using constant values. 

In supervised machine learning, the typical task is to estimate a target variable from other variables. To account for this distinction in our setup, we call one variable $Y\in \mathcal{V}$ the \emph{target} and the remaining variables $\mathcal{F} = \mathcal{V}\backslash\{Y\}$ \emph{features}. The features will be indexed as $\mathcal{F}=\{X_i: \,1\leq i \leq n\}$.

\subsection{Collider bias}
\label{subsec:collider_bias}
\paragraph{Suppressor variables and collider bias.}
In Figure \ref{fig:shapley_values_diabetes_breakfast}, we observed that even for simple problems, uninformative features, i.e., features that are independent of the target, can be attributed importance by XAI methods. This is because these features might be useful to remove variance from informative features. In the literature, variables that are useful in such a secondary sense are known as \emph{suppressor variables}, as they can be used to suppress noise \citep{horst1941role,conger1974revised,darlington1968multiple,kim2019causal}. These references all give different definitions of what a suppressor is, partially contradicting each other or giving definitions that only make sense in a linear framework. We here follow \cite{wilming2022scrutinizing} and \cite{haufe2024explainable} and define suppression through \emph{collider bias}, the statistical association that arises due to the unblocking of paths when conditioning on a collider or its ancestors (cf. Section \ref{subsec:basics_of_causal_inference}). We call a (possibly informative) feature a \emph{suppressor} if it is connected to the target via paths that contain colliders. In this view, suppression and collider bias are just two sides of the same coin: suppressors are those variables whose association with the target might be affected by collider bias when conditioning on other variables.
 Example \ref{ex:diabetes_breakfast} with the causal graph from Figure \ref{fig:causal_graph_diabetes_breakfast} is a very easy example on how the collider $G$ makes $C$ a suppressor. For general setups and a singleton $\mathcal{S}=\{X_k\}$ the middle column of Table \ref{tab:behavior_cond_and_do} summarizes how conditioning on $X_k$ can lead to collider bias between a feature $X_j$ and the target, making $X_j$ a suppressor. The rows of Table \ref{tab:behavior_cond_and_do} distinguish various types of paths between $X_j$ and $Y$. In cells marked in red, conditioning on $X_k$ will (potentially) lead to collider bias between $X_j$ and $Y$: previously blocked paths might become unblocked, inducing statistical associations along these paths. As the number of paths quickly grows with $|\mathcal{F}|$, the consequences of spurious association introduced by such effects is likely to become aggravated even for low-dimensional problems.

\paragraph{Univariate importance is insufficient.}

The problem of suppressor variables and the susceptibility of many XAI methods to them is known in the XAI literature \citep{wilming2022scrutinizing,haufe2024explainable,konig2025disentangling,weichwald2015causal}. \cite{haufe2014interpretation} observed that even the weights of linear models will highlight suppression variables as relevant. \cite{clark2025correcting} and \citet{gjolbye2025minimizing} extended this observation to generalized additive models and locally linear explanations. These three works also propose potential remedies for this problem, which could, in a nutshell, be summarized as (re-)fitting the considered feature $X_j$ to the target $Y$ in a univariate fashion\footnote{\cite{haufe2014interpretation} actually propose a method that, under the assumption of Gaussianity, is equivalent to a univariate linear regression from target to feature. The regression coefficient of such a fit equals, up to standard deviations, a linear univariate regression from feature to target.}. In the notation of \eqref{eq:obs_importance}, we can express such an approach as 
\begin{align}
    \label{eq:univariate_quantitiy}
    I_{\emptyset}(X_j)=\EE[Y|X_j] - \EE[Y] \,,
\end{align}
where we recall that the conditional expectation is the optimal fit of $X_j$ onto $Y$, cf. Appendix \ref{app:conditional_expecation_as_optimal_model}. The object \eqref{eq:univariate_quantitiy} is not susceptible to collider bias and thus to suppression as it does not condition on other features. In particular, we have 
 $X_j \ind Y \Rightarrow I_{\emptyset}(X_j)=0$ so that \eqref{eq:univariate_quantitiy} only highlights features that are statistically associated with $Y$, a property called the \emph{statistical association property} (SAP) by \cite{clark2025correcting,wilming2023theoretical} and \cite{haufe2024explainable}. For Example \ref{ex:diabetes_breakfast}, the SAP guarantees that $I_{\emptyset}(C)=0$.

Univariate importance measures such as \eqref{eq:univariate_quantitiy} avoid attributing importance to statistically irrelevant features which is necessary to apply XAI for model debugging and scientific discovery. However, the interplay between variables is vital to the performance of a ML model and interactions of variables can create information that is not available by looking at each variable on its own as shown by
\begin{example}[Univariate importance is not enough]
\label{ex:univariate_importance_not_enough}
Let $X_1,X_2\sim \mathrm{Bernoulli}(\frac{1}{2})$ be independent and $Y:=X_1 \cdot X_2$. Then $I_{\emptyset}(X_1)=I_{\emptyset}(X_2)=0$ indicate no importance of the features whereas the ``higher order'' importance terms $I_{X_2}(X_1)=X_1X_2-\frac{1}{2}X_2,\,I_{X_1}(X_2)=X_1X_2-\frac{1}{2}X_1 $ (correctly) indicate importance.
\end{example}
A more complete XAI method will thus have to consider the multivariate interplay between variables. As we saw above, doing this purely based on observational data as done by mainstream approaches can lead to spurious association that can easily be misinterpreted. In the following section we demonstrate that using causal information provides a way out of this dilemma.
\begin{table}[t]
\small{
\begin{center}
\renewcommand*{\arraystretch}{1.25}
\begin{tabular}{m{3.75cm}cc}
\toprule
 Paths $\tau$ between $X_j$ and $Y$ & \makecell{$\ldots | X_k$} & \makecell{$\doop(X_k)$} \\
\midrule
$\rightarrow X_k \rightarrow $ exists on $\tau$ & \potblocked $\Rightarrow$ \blocked & \potblocked $\Rightarrow$ \blocked \\
$\leftarrow X_k \rightarrow $ exists on $\tau$ & \potblocked $\Rightarrow$ \blocked  & \potblocked $\Rightarrow$ \blocked  \\
$\rightarrow X_k \leftarrow $ exists on $\tau$ & \cellcolor{pink} \blocked \hspace{0.15em}$\Rightarrow$ \potblocked  & \blocked \hspace{0.05em} $\Rightarrow$ \blocked \\
$\rightarrow C \leftarrow$ exists on $\tau$ and $X_k\notin \tau,C\in \ancestors(X_k)$ & \cellcolor{pink} \blocked \hspace{0.15em}$\Rightarrow$ \potblocked & \blocked \hspace{0.0em} $\Rightarrow$ \blocked \\
None of the above & no effect & no effect \\
\bottomrule
\end{tabular}
\end{center}}
\caption{Change in the state of being blocked (\blocked) or unblocked (\unblocked) for paths between feature $X_j$ and target $Y$ when conditioning (middle column) or intervening (right column) on $X_k$. "\potblocked" means that the path is potentially unblocked - depending on other path details.
Conditioning on $X_k$ potentially unblocks previously blocked paths, which can lead to \emph{collider bias} (cells marked in red). This is avoided by causal interventions $\doop(X_k)$.} \label{tab:behavior_cond_and_do}
\end{table}

\section{Methodology}
\label{sec:methodology}
\subsection{Using the causal context}
\label{subsec:using_the_causal_context}

\paragraph{A causal version of Shapley values.}%
We saw in Section \ref{sec:introduction} for our running Example \ref{ex:diabetes_breakfast} that, even in this simple case, objects such as $\EE[Y|C,G]$ are susceptible to the problem of spurious correlations, namely between $C$ and $Y$ conditional on $G$.
We here propose to solve problems of this kind as described in the following.

First, abandoning the symmetry between the considered features, we propose to distinguish in each scenario between 
\begin{enumerate}
\item The variable whose importance is to be studied. 
\item The variable(s) that describe the background within which the importance of the variable from 1 is studied. We call these variables the \emph{context}.
\end{enumerate}
The ``paradoxical'' behavior for Example \ref{ex:diabetes_breakfast} arose when studying the importance of the feature $C$ (carbohydrate intake) in the context of a glucose value of $G$. The quantity $\EE[Y|C,G]$ considers the collider $G$ as an observation, which leads to the spurious correlation between the suppressor $C$ and $Y$. We here propose to apply \emph{interventions} on the context variable(s) instead of conditioning on observed values. 

\begin{definition}
\label{def:cc_Shapley}
We define the \emph{importance of a feature $X_j$ in the interventional context of features $\mathcal{S}\subseteq \mathcal{F}\backslash\{X_j\}$} as 
\begin{align}
    \label{eq:importance_do}
    I_{\doop(\mathcal{S})}(X_j)= \EE[Y|X_j,\doop(\mathcal{S})] - \EE[Y|\doop(\mathcal{S})] \,.
\end{align}
Using \eqref{eq:importance_do} we define causal context Shapley (\emph{cc-Shapley}) values as
\begin{align}
    \label{eq:shapley_values_do}
    \phi_{cc}(X_j) = \!\!\!\!\sum_{\mathcal{S}\subseteq \mathcal{F}\backslash\{X_j\}} \!\!\!\! {\small \frac{|\mathcal{S}|! (|\mathcal{F}|-|\mathcal{S}|-1)!}{|\mathcal{F}|!}} I_{\doop(\mathcal{S})}(X_j)\,.
\end{align}

\end{definition}
Intervening on the context allows us to remove collider bias: For the case of a univariate context $\mathcal{S}=\{X_k\}$, the rightmost column of Table \ref{tab:behavior_cond_and_do} summarizes the effect of the intervention $\doop(X_k)$ on the blocking of paths between $X_j$ and $Y$. In contrast to a mere conditioning on $X_k$ (middle column), the intervention $\doop(X_k)$ does not unblock previously blocked paths. 
In fact, the same holds for general context variables $\mathcal{S}$:
Since intervention removes incoming arrows to $\mathcal{S}$ in $\mathcal{G}$, the operation $\doop(\mathcal{S})$ does not lead to a conditioning on colliders. 

For the running Example \ref{ex:diabetes_breakfast}, we obtain, from Lemmas \ref{lem:irrelevant_context}, \ref{lem:observation_equals_intervention} below and Appendix \ref{app:details_on_diabetes_breakfast}, $I_{\doop(C)}(G)=I_{C}(G)\simeq \sigma(\frac{2}{5}(G-0.4C-109))-0.15$ and $I_{\doop(G)}(C)=I_{\emptyset}(C)=0$. Note that the information on the suppressor $C$ is still used but now counts towards the importance of the informative variable $G$ it suppresses. The uninformative suppressor feature $C$ is assigned no importance. This can be shown to hold in general for any non-informative feature whenever the underlying causal graph is faithful. Thus, the cc-Shapley values as defined in Definition \ref{def:cc_Shapley} fulfill the SAP.

\begin{proposition}[SAP for Definition \ref{def:cc_Shapley}]
\label{prop:sap_I_do}
Given a feature $X_j$ with $X_j \ind_{\mathcal{G}} Y$ we have $I_{\doop(\mathcal{S})}(X_j)=0$ for any $\mathcal{S}\subseteq \mathcal{F}\backslash\{X_j\}$. This implies that we have $X_j \ind_{\mathcal{G}} Y \Rightarrow \phi_{cc}(X_j)=0$ for $\phi_{cc}$ as in \eqref{eq:shapley_values_do}. 
\end{proposition}
\begin{proof}
This follows from the fact that an intervention on $\mathcal{S}$ cannot unblock paths between $X_j$ and $Y$, hence $X_j\ind_{\mathcal{G}} Y |\,\doop(\mathcal{S})$ and thus $X_j \ind  Y |\,\doop(\mathcal{S})$. 
\end{proof}

\paragraph{Comparison with other methods.}
\cite{jung2022measuring} introduce \emph{$\doop$-Shapley} via the formula
\begin{align*}
    \phi_{\doop}(X_j) =\!\!\!\! \sum_{\mathcal{S}\subseteq \mathcal{F}\backslash\{X_j\}} \!\!\!\!\!\!\! \gamma(\mathcal{S}) (\EE[Y|\doop(X_j,\mathcal{S})]-\EE[Y|\doop(X_j)]\,.
\end{align*}
This expression appears similar to the one proposed for $\phi_{cc}$ above, with the exception that we do not intervene on $X_j$ in Definition \ref{def:cc_Shapley}. This difference is however crucial for the aspects studied in this work. Broadly speaking, prediction problems can often be classified as causal and anti-causal, depending on whether the target is the effect or the cause of the features \citep{scholkopf2012oncausal}. While the issue of suppression is not exclusive to the anti-causal setup, it most naturally occurs in such settings due to its connection to collider bias. In fact, almost all causal settings studied in this work (except for Example \ref{ex:univariate_importance_not_enough}) contain at least some anti-causal elements. The object $\phi_{\doop}$ is not designed for anti-causal settings and doesn't attribute relevance to any feature when used in this context (cf. Appendix \ref{app:detailed_comparision_with_other_methods}). The same is true for the \emph{intrinsic causal contribution} approach proposed by \cite{janzing2024quantifying}. \cite{heskes2020causal} propose \emph{causal Shapley values} that differ from do-Shapley via using the actual model instead of $Y$. We find in Appendix \ref{app:detailed_comparision_with_other_methods} that this leads for Example \ref{ex:diabetes_breakfast} once more to spurious negative relevance attributed to $C$. \cite{frye2020asymmetric} also use the model for their concept of \emph{asymmetric Shapley values}. We find in Appendix \ref{app:detailed_comparision_with_other_methods} that their approach can, however, be modified to a model agnostic form and that this form is, similar to cc-Shapley, well-behaved on Example \ref{ex:diabetes_breakfast}. In the presence of mediators this modification shows however inconsistencies in its behavior on causal and anti-causal problems in contrast to cc-Shapley. For a more detailed discussion we refer to Section \ref{app:detailed_comparision_with_other_methods} of the appendix.


\subsection{Estimation}
\label{subsec:estimation}

\paragraph{Simple cases.}
Objects such as $\EE[Y|\mathcal{S}]$ and $\EE[Y|X_j,\mathcal{S}]$ can rather easily be learned from data via training a machine learning model of choice that predicts $Y$ from $\mathcal{S}$ or $\{X_j\}\cup\mathcal{S}$, cf. Lemma \ref{lem:conditional_expectation_as_optimal_model} in the appendix. Interventional objects such as $\EE[Y|X_j,\doop(\mathcal{S})]$ from Definition \ref{def:cc_Shapley} are harder to obtain. But in some cases they match observational quantities as shown by the following two lemmas whose proofs are in Appendix \ref{app:proofs_of_lemmas}.

\begin{lemma}[Irrelevant context]
\label{lem:irrelevant_context}
Consider $X_j\in \mathcal{F}, \mathcal{S}\subseteq \mathcal{F}\backslash\{X_j\}$ and assume that there are no causal paths $\mathcal{S} \ipath Y, X_j$. Then we have $\EE[Y|X_j, \doop(\mathcal{S})]=\EE[Y|X_j]$ and $I_{\doop(\mathcal{S})}(X_j)=I_{\emptyset}(X_j)$.  
\end{lemma}

For the running Example \ref{ex:diabetes_breakfast}, Lemma \ref{lem:irrelevant_context} implies that $I_{\doop(G)}(C)=I_{\emptyset}(C)=0$ as there are no causal paths $G\ipath Y$ or $G\ipath C$ and hence $G$ is irrelevant as context for $C$.

\begin{lemma}[Intervention equals observation]
\label{lem:observation_equals_intervention}
Consider $X_j\in \mathcal{F},\,\mathcal{S}\subseteq \mathcal{F}\backslash\{X_j\}$ such that, either
\begin{itemize}
    \item \emph{(no backdoor paths)} there are no unblocked backdoor paths from $\mathcal{S}$ to $X_j$ or from $\mathcal{S}$ to $Y$, or
    \item \emph{(purely causal setup)} there are no causal paths $Y \ipath X_j,\mathcal{S}$ and no confounders $H \in \mathcal{V}\backslash \left(\{Y,X_j\}\cup\mathcal{S}\right)$ with $X_j \pathi H \ipath Y$ or $\mathcal{S} \pathi H \ipath Y$,
\end{itemize}

then we have the identities $\EE[Y|X_j,\doop(\mathcal{S})] = \EE[Y|X_j,\mathcal{S}]$ and $I_{\doop(\mathcal{S})}(X_j)=I_{\mathcal{S}}(X_j)$.
\end{lemma}

The feature $C$ in Example \ref{ex:diabetes_breakfast} has no backdoor paths so that Lemma \ref{lem:observation_equals_intervention} implies $I_{\doop(C)}(G)=I_C(G)$. The same argument implies that for Example \ref{ex:univariate_importance_not_enough} we have $I_{\doop(X_2)}(X_1)=I_{X_2}(X_1)$ and $I_{\doop(X_1)}(X_2)=I_{X_1}(X_2)$ and thus $\phi_{cc}(X_1)=\phi(X_2)\neq 0,\phi_{cc}(X_2)=\phi(X_2)\neq 0$.

\paragraph{Backdoor adjustment.} 
In Lemma \ref{lem:backdoor_adjustment} in the appendix we provide a formula for the identification of $I_{\doop(\mathcal{S})}(X_j)$ from observational data similar to the backdoor adjustment from \cite{pearl2009causality}.

\begin{algorithm}[t]
\caption{Compute $I_{\doop(\mathcal{S})}(X_j)$ from SCM}
\label{alg:compute_importance}
\KwIn{SCM $\mathcal{M}$, $X_j \in \mathcal{F}$,
      context $\mathcal{S} \subseteq \mathcal{F}\backslash\{X_j\}$}
\KwOut{$I_{\doop(\mathcal{S})}(X_j)$}
\vspace{0.4em}  
\Comment{Create modified model}
Create sampler of marginal $q(\mathcal{S})$ in $\mathcal{M}$\;
Construct $\mathcal{M}^{\doop(\mathcal{S}\sim q)}$\;

\vspace{0.3em} 
\Comment{Fit ML models (NN, XGBoost,...)}
Sample $(X_j,\mathcal{S},Y)$ from $\mathcal{M}^{\doop(\mathcal{S}\sim q)}$\;
$\EE[Y|\doop(\mathcal{S})] \gets $ fit $\mathcal{S}$ to $Y$\;
$\EE[Y|X_j,\doop(\mathcal{S})] \gets $ fit $X_j, \mathcal{S}$ to $Y$\;

\vspace{0.3em}  
\Comment{Compute importance}
$I_{\doop(\mathcal{S})}(X_j) \gets \EE[Y|X_j,\doop(\mathcal{S})]- \EE[Y|\doop(\mathcal{S})]$ \;
\vspace{0.4em} 
\Return $\EE[Y|X_j,\doop(\mathcal{S})]$, $I_{\doop(\mathcal{S})}(X_j)$
\end{algorithm}
\paragraph{Using the SCM.}
Algorithm \ref{alg:compute_importance} summarizes how $I_{\doop(\mathcal{S})}(X_j)$ can be computed given an SCM $\mathcal{M}$. This is the method that we use in Section \ref{sec:experimental_results} below. We first isolate the joint marginal $q$ of the context variables $\mathcal{S}$ within the original model $\mathcal{M}$. We then perform a stochastic intervention on the SCM to obtain $\mathcal{M}^{\doop(\mathcal{S}\sim q)}$. In this modified model, we can obtain estimates of the functions $\EE[Y|X_j, \doop(\mathcal{S})]$ and $\EE[Y|\doop(\mathcal{S})]$ via standard multivariate fitting using data-driven models, cf. Appendix \ref{app:justification_of_algorithm}.

Unless one is working with synthetic examples, the SCM will naturally not be known, as we also discuss in Section \ref{sec:limitations} below. 
 However, given a causal graph and a defined noise structure, the functional relations can be estimated from observational data by regressing each feature on its parents \citep{peters2017elements}. We apply this for the real world example in Section \ref{sec:experimental_results} below.

\subsection{Limitations}
\label{sec:limitations}
Neither the original Shapley values \eqref{eq:shapley_values} nor our causal analogue \eqref{eq:shapley_values_do} are scalable in the way they are formulated here. The number of of context sets $\mathcal{S}$ increases swiftly with $|\mathcal{F}|$ and for each feature $X_j$ and context $\mathcal{S}$ two data-driven models have to be fitted to estimate $\EE[Y|X_j, \doop(\mathcal{S})]$ and $\EE[Y|\doop(\mathcal{S})]$. Scalable approximations \citep{chen2023algorithms,parafita2025practical,teal2026exactly} could be a way out but were not studied in the scope of this work. The main focus of this article is to highlight a systematic weakness in the existing approach to XAI and to point out what is necessary to fix it. 

Another limitation that is shared by many other works on causality, is the assumption that the causal graph that has generated the data is known. Once we have obtained this graph, we can often fit the functions $(f_X)_{X\in \mathcal{V}}$ in the SCM and subsequently apply Algorithm \ref{alg:compute_importance}. This is the procedure that we follow for the real world example in Section \ref{sec:experimental_results} below. 
In specific cases, such as linear SCMs with non-Gaussian noise, algorithms such as LiNGAM \citep{shimizu2011directlingam} can be used, cf. Section \ref{sec:experimental_results} below, to obtain the causal graph. For more general approaches to causal discovery compare also the work of \cite{zheng2018dags} and the reviews by \cite{zanga2022survey,glymour2019review}. In the general case, however, obtaining a valid causal graph is typically non-trivial and often requires expert knowledge. Moreover, in accordance with many other works in causal inference, we have to assume that the exogenous variables in our model are uncorrelated, which is typically an oversimplification.

Finally, the usage of a causal graph assumes that a variable has a specific static meaning throughout the dataset, which is typically not true for other types of data such as images. Here causal representation learning becomes necessary to apply causal techniques \citep{scholkopf2021toward,brehmer2022weakly,ahuja2023interventional}. We did not use such techniques here and only restrict ourselves to static data.

\section{Experimental results}
\label{sec:experimental_results}
\begin{figure}
    \centering
  \begin{subfigure}[t]{0.235\textwidth}
    \includegraphics[width=\linewidth]{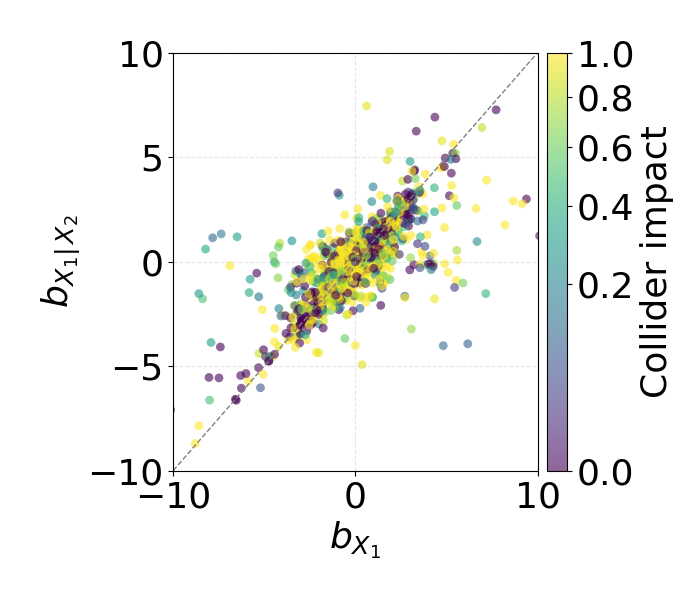}
  \end{subfigure}
  \begin{subfigure}[t]{0.235\textwidth}
    \includegraphics[width=\linewidth]{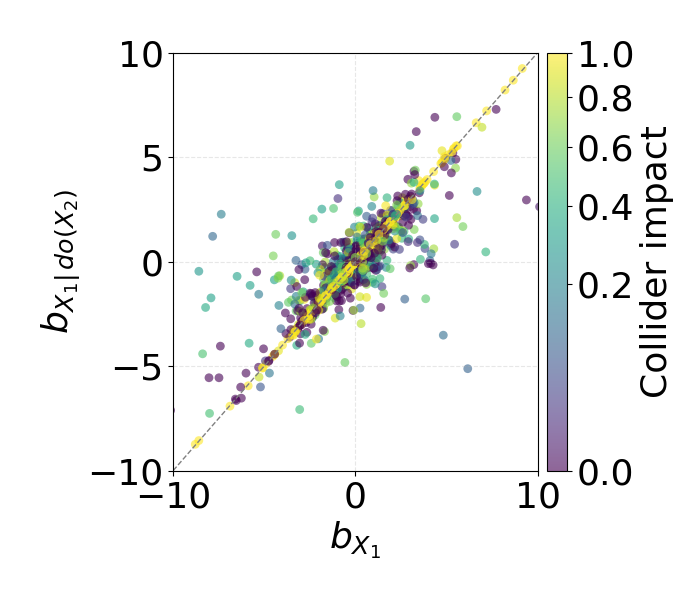}
   \end{subfigure}
    \caption{Comparison of the regression coefficients $b_{X_1|X_2}$ (left) and $b_{X_1|\doop(X_2)}$ (right) with $b_{X_1}$ for randomly sampled linear SCMs with 9 variables. The color encodes the extent to which $X_2$ acts as collider, cf. Appendix \ref{app:details_on_linear_SCM_experiment}.}
    \label{fig:linear_scm_experiment}
\end{figure}

\paragraph{Linear SCMs.}
We randomly sample 3,000 linear SCMs with 8 features, i.e., $\mathcal{F}=\{X_1,\ldots,X_8\},\,\mathcal{V}=\mathcal{F}\cup \{Y\}$, and additive non-Gaussian (Laplacian) noise as described in Appendix \ref{app:details_on_linear_SCM_experiment}. For each SCM, the effect of including univariate context $\mathcal{S}=\{X_2\}$ on the importance of the variable $X_1$ is studied. 
Using linear models for the fitting of conditional expectations and for Algorithm \ref{alg:compute_importance}, we obtain partial regression coefficients $b_{X_1|X_2}, b_{X_1|\doop(X_2)}, b_{X_1}$ for the $X_1$-dependency of $I_{X_2}(X_1)$, $I_{\doop(X_2)}(X_1)$ and $I_{\emptyset}(X_1)$.

Figure \ref{fig:linear_scm_experiment} compares $b_{X_1|X_2}$ vs. $b_{X_1}$ (left) and  $b_{X_1|\doop(X_2)}$ vs. $b_{X_1}$ (right) for all sampled SCMs. Each point represents a different linear SCM. For the computation of $b_{X_1|\doop(X_2)}$ we used the DirectLiNGAM algorithm from \cite{shimizu2011directlingam} to identify the causal graph, so that no prior knowledge about the causal structure was used. The color of each point represents how many of the paths between $X_1$ and $Y$ that run through $X_2$ are blocked due to $X_2$ acting as a collider. This is measured by a heuristic quantity described in Appendix \ref{app:details_on_linear_SCM_experiment}.

For yellow points, almost all paths between $X_1$ and $Y$ running through $X_2$ contain $X_2$ as a collider, so that we expect the collider bias to be most distinct for these points. We observe that, in the comparison $b_{X_1|X_2}$ vs. $b_{X_1}$ (left), these points are usually placed far away from the diagonal, indicating that including the observed value of $X_2$ did change the importance of $X_1$. For the interventional comparison $b_{X_1|\doop(X_2)}$ vs. $b_{X_1}$ (right), these points lie on the diagonal and hence context $X_2$ that mostly affects $X_1$ via collider bias does not lead to a change in importance attributed to $X_1$.

\paragraph{Nonlinear case.} 
\begin{figure}[t]
    \centering
    \includegraphics[width=\linewidth]{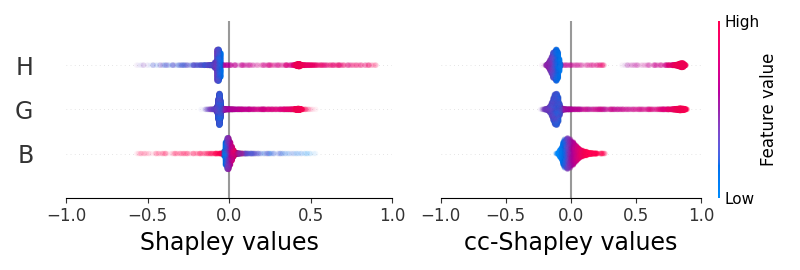}
    \caption{Shapley (left) and cc-Shapley (right) values for the nonlinear diabetes example considered in Section \ref{sec:experimental_results}.}
    \label{fig:shapley_values_diabetes_risk}
\end{figure}
Consider again the task of predicting whether a patient has type 2 diabetes ($Y=1$) or not ($Y=0$). We assume this time that the patient did obey and not eat breakfast, but measure in addition to the blood glucose $G$ also the average blood sugar $H$ and the BMI $B$ of the patient.
The causal graph for this example is shown in Figure \ref{fig:causal_graph_diabetes_risk}, the (synthetic) SCM is specified in Appendix \ref{app:details_on_diabetes_risk}.

\begin{figure}[t]
\setlength{\tabcolsep}{6pt}
\renewcommand{\arraystretch}{1.2}

\begin{tabular}{c C{0.22\textwidth} C{0.22\textwidth}}

\multirow{2}{*}{\vspace{-2.6cm}\rotatebox{90}{\textit{univariate}}}
 & $I_{\emptyset}(B)$
 & $I_{\emptyset}(G)$ \\

 & \hspace{-0.8cm}\includegraphics[width=0.85\linewidth]{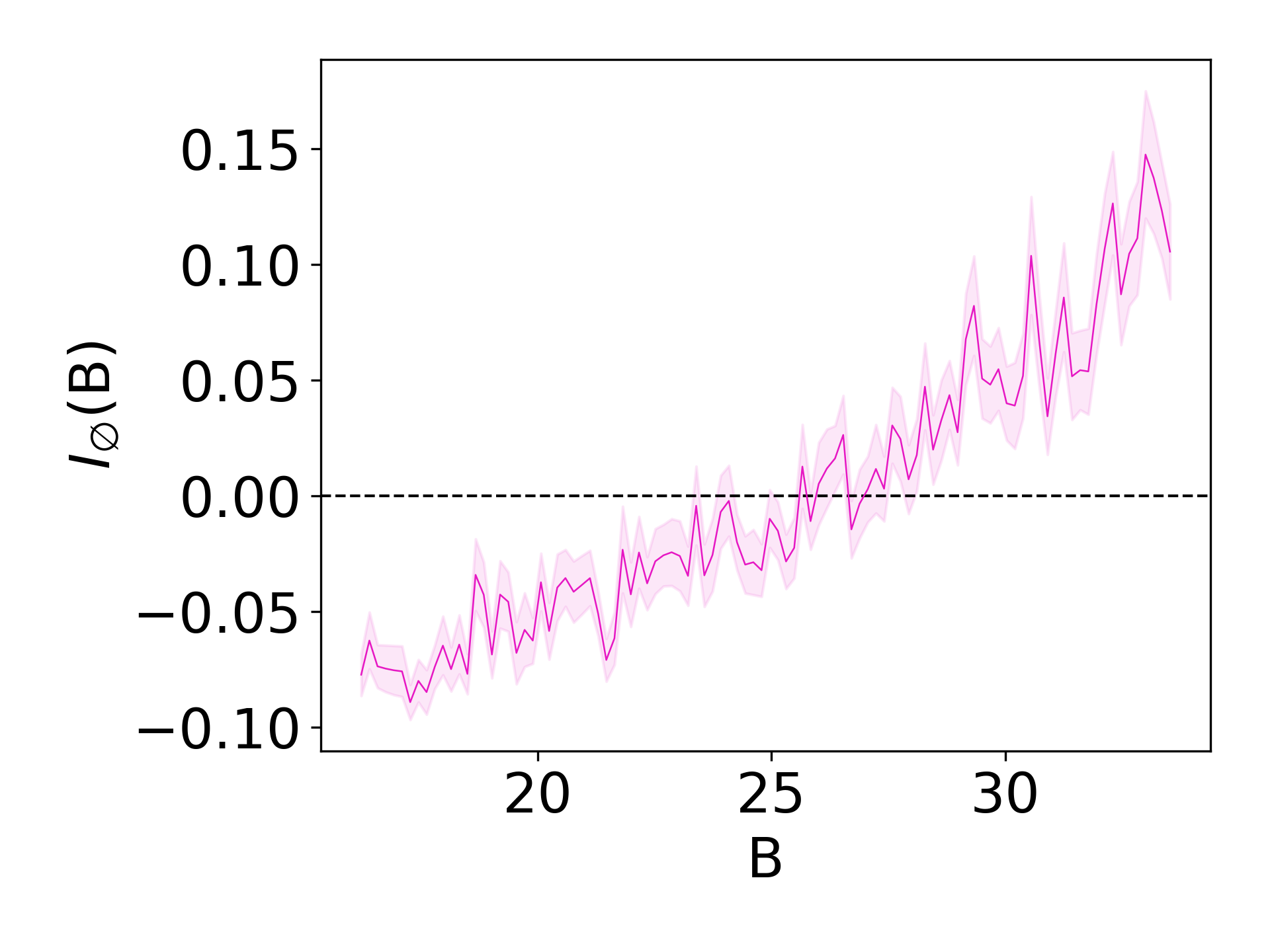}
 & \hspace{-0.4cm}\includegraphics[width=0.85\linewidth]{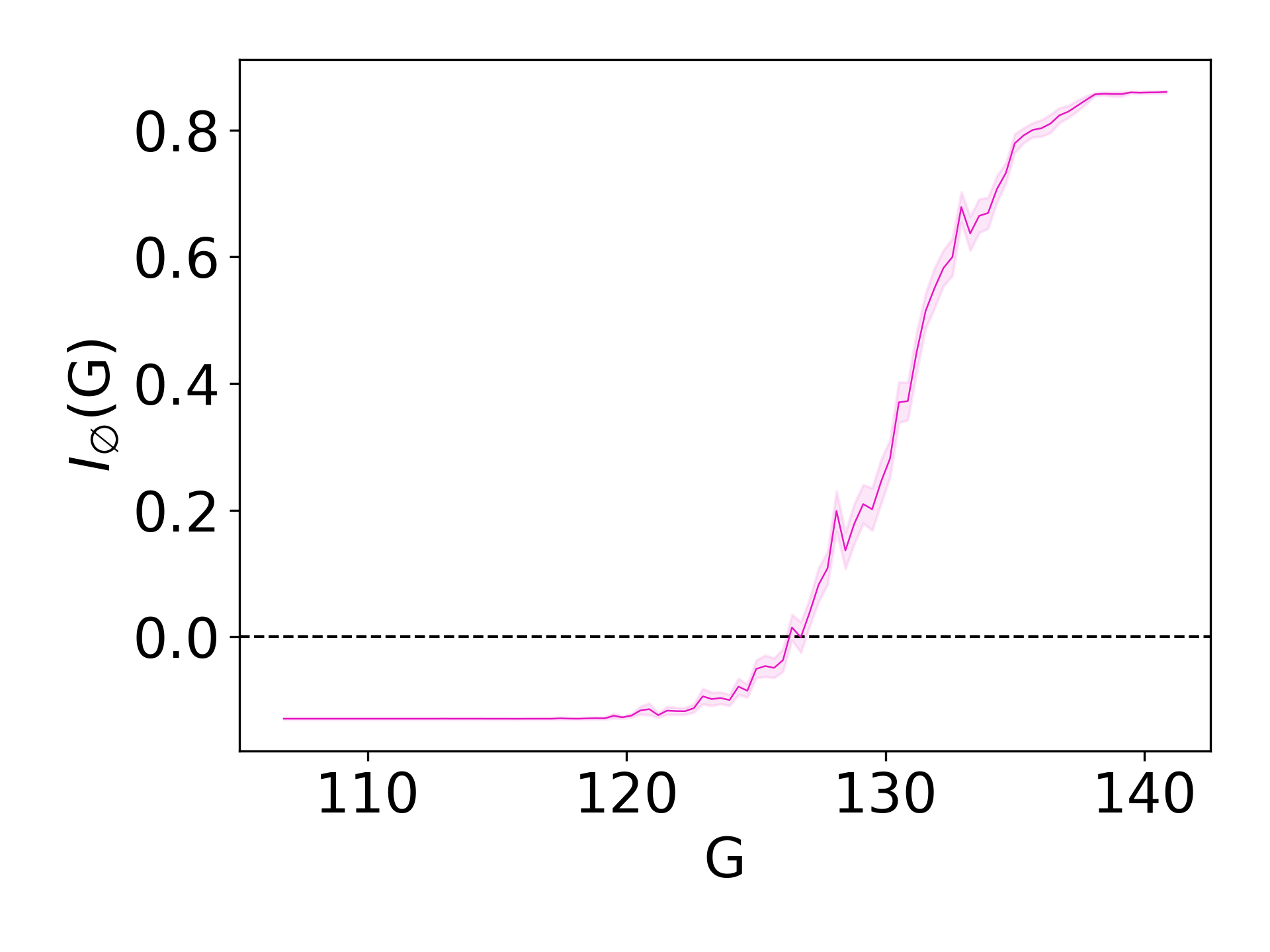}\\ 

\multirow{2}{*}{\vspace{-2.7cm}\rotatebox{90}{\textit{observational}}}
 & $I_G(B)$
 & $I_B(G)$ \\

 & \includegraphics[width=\linewidth]{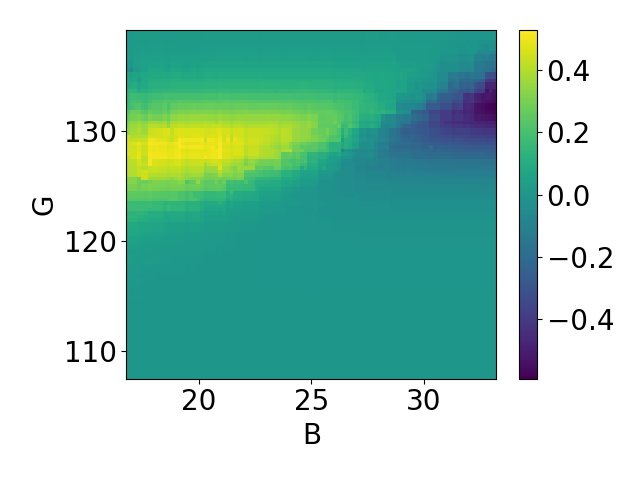}
 & \includegraphics[width=\linewidth]{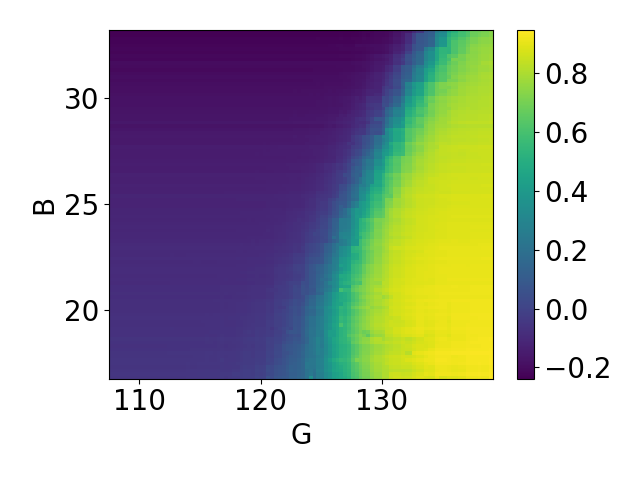}\\ 

\multirow{2}{*}{\vspace{-2.7cm}\rotatebox{90}{\textit{interventional}}}
 & $I_{\doop(G)}(B)$
 & $I_{\doop(B)}(G)$ \\

 & \includegraphics[width=\linewidth]{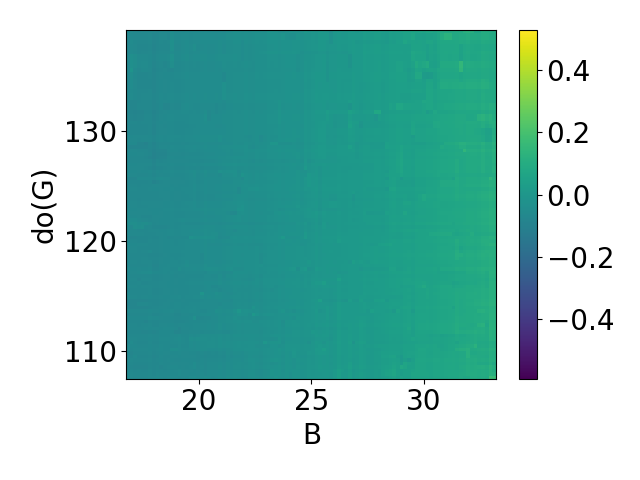}
 & \includegraphics[width=\linewidth]{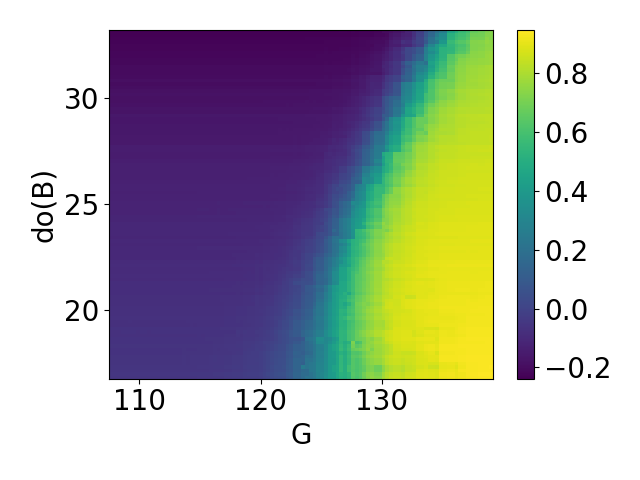} 
\end{tabular}
\caption{Various terms arising in the computation of the conventional observational Shapley values \eqref{eq:shapley_values} and the cc-Shapley values \eqref{eq:shapley_values_do} for the (more complex) diabetes example introduced in Section \ref{sec:experimental_results}. The upper row shows (with standard error) the context-free univariate term that coincides for \eqref{eq:shapley_values} and \eqref{eq:shapley_values_do}. 
}
\label{fig:bivariate_shape_functions}
\end{figure}
It is widely accepted that a high BMI $B$ is associated with an increased risk of type 2 diabetes \citep{chandrasekaran2024role}. 
When looking at the Shapley values \eqref{eq:shapley_values} shown on the left in Figure \ref{fig:shapley_values_diabetes_risk}, we might come to a different conclusion:  The BMI $B$ is attributed a negative relevance for diabetes, which one could easily misinterpret as ``high BMI leads to a low diabetes risk''.  Using cc-Shapley values \eqref{eq:shapley_values_do} instead leads to the values depicted on the right in Figure \ref{fig:shapley_values_diabetes_risk}. The importance behavior of $B$ (and the other features) now matches intuition. This different behavior of Shapley and cc-Shapley values arises, once more, due to collider bias. 

 The Shapley values for $B$ in this example are computed as $\phi(B)=\frac{1}{3}I_{\emptyset}(B)+\frac{1}{6}I_{G}(B)+\frac{1}{6}I_{H}(B)+\frac{1}{3}I_{\{G,H\}}(B)$. 
As shown in the first plot in the top row of Figure \ref{fig:bivariate_shape_functions}, the context-free quantity $I_{\emptyset}(B)$ alone indicates a positive relevance for diabetes risk in line with our expectations. The overall negative relevance of large $B$ must therefore stem from the terms with context. The left plot in the middle row of Figure \ref{fig:bivariate_shape_functions} shows the values of $I_{G}(B)$ for various choices of $B$ and $G$ as a heatmap. For $G$ in a certain range, $I_{G}(B)$ indicates a negative relevance of $B$. A similar behavior can be observed for $H$, cf. Appendix \ref{app:details_on_diabetes_risk}. A look at the causal graph in Figure \ref{fig:causal_graph_diabetes_risk} reveals that $H$ and $G$ act as colliders between $B$ and $Y$ so that conditioning on $G$, $H$ or $\{G,H\}$ will create a spurious (here negative) association with $Y$ leading to the pattern observed in Figure \ref{fig:shapley_values_diabetes_risk}.

Once we use $I_{\doop(G)}(B)$ instead of $I_{G}(B)$, this bias is removed and we observe no negative relevance of $B$ in the context of $G$, as shown by the first plot in the third row of Figure \ref{fig:bivariate_shape_functions}.  Interestingly, when we switch the roles of $B$ and $G$, we obtain a different effect: Lemma \ref{lem:observation_equals_intervention} implies that $I_{\doop(B)}(G)=I_{B}(G)$ as there are no backdoor paths to $B$. Both observational and interventional objects now match, cf. the second column of Figure \ref{fig:bivariate_shape_functions}, and reveal a positive relevance of $G$. The reason for the context dependency of $I_{\doop(B)}(G)$ is that the relation $B \rightarrow G$ is causal and knowing $B$
allows us to remove the variance in $G$ and thus judge the relation between $G$ and $Y$ with more precision. This behavior resembles the observations of Section \ref{sec:methodology}: The suppression effect of $B$ on $G$ is now only accounted for in the relevance of the suppressed variable $G$ and does not affect the importance of the suppressor $B$.
In fact, we have $\phi_{cc}(B)=I_{\emptyset}(B)$ as we show in Appendix \ref{app:details_on_diabetes_risk}.

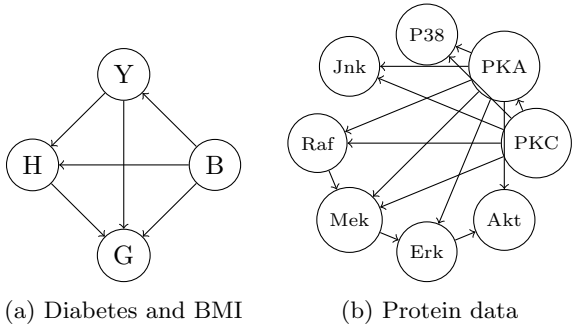
\begin{figure}[t]
  \begin{subfigure}[t]{0.23\textwidth}
    \centering
  \begin{tikzpicture}[scale=0.6, every node/.style={circle, draw}]
      \draw
        (0.0:2) node (B){B}
        (90.0:2) node (Y){Y}
        (180.0:2) node (H){H}
        (270.0:2) node (G){G};
      \begin{scope}[->]
        \draw (B) to (Y);
        \draw (B) to (H);
        \draw (B) to (G);
        \draw (Y) to (H);
        \draw (Y) to (G);
        \draw (H) to (G);
      \end{scope}
    \end{tikzpicture}

    \caption{Diabetes and BMI}
    \label{fig:causal_graph_diabetes_risk}
  \end{subfigure}
  \begin{subfigure}[t]{0.23\textwidth}
    \centering
  \begin{tikzpicture}[scale=0.72,every node/.style={circle, draw,font=\scriptsize}]
      \draw
        (0.0:2) node (PKC){PKC}
        (45.0:2) node (PKA){PKA}
        (90.0:2) node (P38){P38}
        (135.0:2) node (Jnk){Jnk}
        (180.0:2) node (Raf){Raf}
        (225.0:2) node (Mek){Mek}
        (270.0:2) node (Erk){Erk}
        (315.0:2) node (Akt){Akt};
      \begin{scope}[->]
        \draw (PKC) to (PKA);
        \draw (PKC) to (P38);
        \draw (PKC) to (Jnk);
        \draw (PKC) to (Raf);
        \draw (PKC) to (Mek);
        \draw (PKA) to (P38);
        \draw (PKA) to (Jnk);
        \draw (PKA) to (Raf);
        \draw (PKA) to (Mek);
        \draw (PKA) to (Erk);
        \draw (PKA) to (Akt);
        \draw (Raf) to (Mek);
        \draw (Mek) to (Erk);
        \draw (Erk) to (Akt);
      \end{scope}
    \end{tikzpicture}
    \caption{Protein data}
    \label{fig:causal_graph_sachs}
    \end{subfigure}
    \caption{Causal graphs for two datasets used in Section \ref{sec:experimental_results}. For the diabetes example (left) the meaning of the variables are: blood glucose $G$, average sugar $H$, BMI $B$ and presence of type 2 diabetes $Y$ (target). For the example from \cite{sachs2005causal} (right) the variables are various proteins with known causal graph. The protein \emph{PKA} is here chosen as target. } 
    \label{eq:causal_graphs}
\end{figure}

\paragraph{A real world example.}
We consider the dataset from \cite{sachs2005causal}, which contains information on the concentration of various proteins. We use a preprocessed version that discretizes the concentration of each protein in three categories, which we label as $0,1$ and $2$.
Details on the dataset and our implementation are contained in Appendix \ref{app:details_on_sachs}.
The causal graph, depicted in Figure \ref{fig:causal_graph_sachs}, shows the causal relation of the 8 considered proteins. We here consider the task of predicting the amount of the protein \emph{PKA}. 

Figure \ref{fig:univariate_importance_sachs_interventional} shows the univariate importance $I_{\emptyset}$ of the three proteins \emph{Jnk, PKC} and \emph{P38} depending on their value in the set \{0,1,2\}. The results for all features are contained in Appendix \ref{app:details_on_sachs}. As we saw above in Example \ref{ex:univariate_importance_not_enough}, univariate importance is in general not sufficient to fully explain relevance. However, we also saw above that we have to beware of collider bias when a fundamental deviation in the relevance behavior only occurs in the observational context of other variables.

Figure \ref{fig:shapley_values_sachs_interventional} shows the Shapley values (left) and cc-Shapley values (right) for the proteins \emph{Jnk, PKC} and \emph{P38}. For \emph{Jnk} both, Shapley and cc-Shapley values, show a behavior that is comparable to the univariate importance shown in Figure \ref{fig:univariate_importance_sachs_interventional}. For \emph{PKC} and \emph{P38} the Shapley values on the left show a deviation to the univariate behavior. Various data points with a high value of \emph{P38} are attributed negative Shapley values. For \emph{PKC} we see a mixed or even negative relevance in contrast to the slight positive relevance in Figure \ref{fig:univariate_importance_sachs_interventional}. Inspecting the causal graph in Figure \ref{fig:causal_graph_sachs} we see that \emph{PKC} (and hence \emph{P38}) is connected with the target \emph{PKA} via various colliders. Using them as observational context as in \eqref{eq:shapley_values} therefore bears the risk of introducing collider bias. Indeed, once we use the context only in a causal manner, we find a relevance attribution that contrasts less with the univariate analysis. In particular the (slight) positive relevance of \emph{PKC} is kept intact by the cc-Shapley values.

\begin{figure}
    \centering
  \begin{subfigure}[t]{0.5\textwidth}
  \centering
    \includegraphics[width=0.5\textwidth]{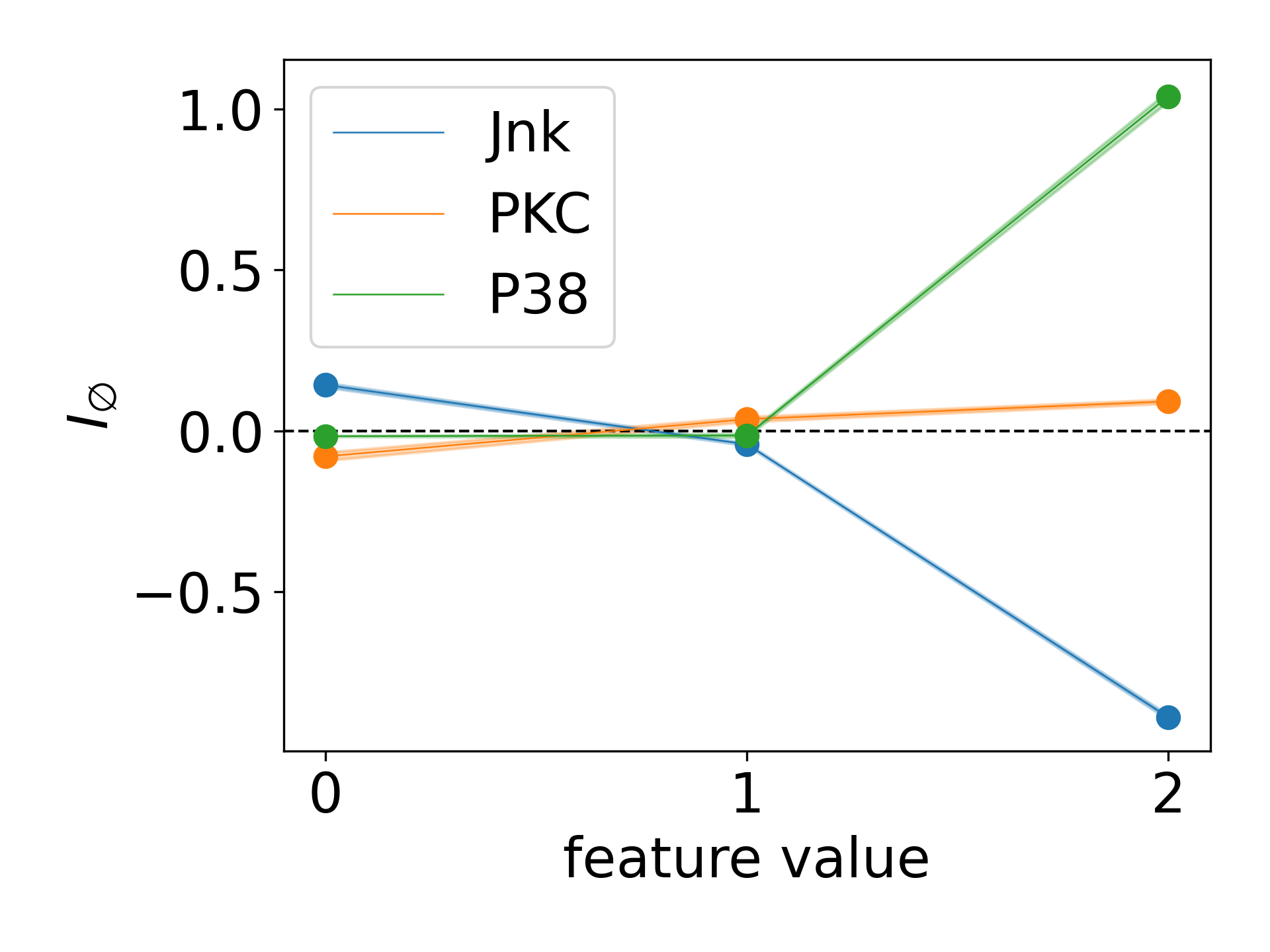}
    \caption{Univariate importance $I_{\emptyset}$}
    \label{fig:univariate_importance_sachs_interventional}
    \end{subfigure}
    \\
    \begin{subfigure}[t]{0.5\textwidth}
    \includegraphics[width=\textwidth]{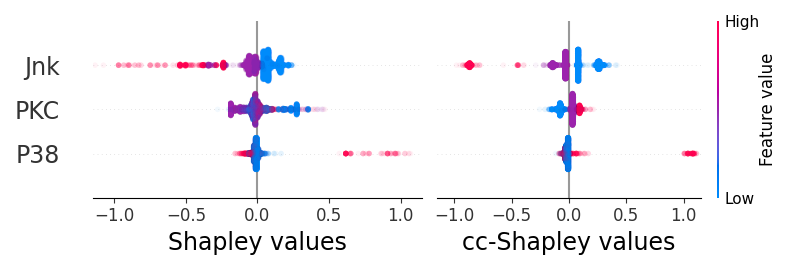}
    \caption{Shapley and cc-Shapley values}
    \label{fig:shapley_values_sachs_interventional}
    \end{subfigure}
    \caption{Some results of the univariate important measure $I_{\emptyset}$ and the Shapley and cc-Shapley values for the data from \cite{sachs2005causal}. Only a subset of the proteins shown in the causal graph \ref{fig:causal_graph_sachs} are depicted here. The results for all proteins are contained in Appendix \ref{app:details_on_sachs}.}
\end{figure}

\section{Conclusion}
In this work, we highlight a blind spot in purely observational approaches to XAI. When studying the relevance of a feature in the observational context of the remaining features, collider bias can distort the feature attribution and even flip the positive relevance of a feature to a negative one and vice versa. This phenomenon can already be observed for hand-crafted two-dimensional problems but also appears for more complex and real world data, and adds to further known biases of XAI methods such as biases towards salient features \citep{clark2026feature}. An incorrect feature attribution can easily lead to misinterpretations and thereby undermine the typical goals pursued when employing XAI.

Eradicating collider bias requires intervention and is hence beyond the reach of the observational rung of Pearl's ladder of causation \citep{pearl2018bookofwhy} on which most XAI methods fall. Given causal knowledge, we propose cc-Shapley values as an approach to correct for spurious associations introduced by collider bias. We observe in theoretical and experimental results that incorrect and misleading attribution due to collider bias are indeed removed when using cc-Shapley values.

\newpage
\bibliography{references}

\newpage
\onecolumn
\appendix 

\section*{\LARGE Supplementary Material}
\maketitle
\section{On the used version of Shapley values}
\label{app:on_the_used_version_of_shapley_values}

\paragraph{Comparison with the literature.} In the context of game theory, \cite{shapley1953value} introduced a general expression of the form
\begin{align}
    \label{eq:appendix_shapley_values_game_theory}
    \phi(X_j) = \sum_{\mathcal{S}\subseteq \mathcal{F}\backslash\{X_j\}} \gamma(\mathcal{S}) (v(\mathcal{S}\cup \{X_j\}) - v(\mathcal{S}))\,,
\end{align}
where $\gamma(\mathcal{S})={\small \frac{|\mathcal{S}|! (|\mathcal{F}|-|\mathcal{S}|-1)!}{|\mathcal{F}|!}}$ and $v: 2^{\mathcal{F}} \rightarrow \RR$ is a general set function that satisfies some simple properties and should be thought of as encoding a game. \cite{lipovetsky2001analysis} use $R^2$ for $v$ in the context of linear regression to judge feature importance. \cite{datta2016algorithmic} consider a more general framework that uses \eqref{eq:appendix_shapley_values_game_theory} to measure the influence of inputs with various choices of $v$. Referring to the works above, \cite{lundberg2017unified} formulate a  ``classic'' version of Shapley values as
\begin{align}
    \label{eq:appendix_shapley_Values_lundberg_classic}
    \phi(X_j) = \sum_{\mathcal{S}\subseteq \mathcal{F}\backslash\{X_j\}} \gamma(\mathcal{S}) (f(\mathcal{S}\cup \{X_j\}) - f(\mathcal{S}))\,,
\end{align}
here $f(\mathcal{S}\cup \{X_j\})$ and $f(\mathcal{S})$ denote functions that were trained only on the features $\mathcal{S}\cup\{X_j\}$ and $\mathcal{S}$ respectively. The optimally trained models $f$ have precise mathematical formulations in terms of a conditional expectation, cf. Lemma \ref{lem:conditional_expectation_as_optimal_model} below. We can thus identify $f(\mathcal{S}\cup \{X_j\})=\EE[Y|\mathcal{S},X_j]$ and $f(\mathcal{S})=\EE[Y|\mathcal{S}]$ and arrive at the expression \eqref{eq:shapley_values} used in this work.

Expression \eqref{eq:shapley_values} is model agnostic in the sense that it only uses optimal models and not a specific instance of a trained model. While \cite{lundberg2017unified} refer to classical Shapley values as formulated in \eqref{eq:appendix_shapley_Values_lundberg_classic}, their proposed algorithm SHAP completely builds on the full model $f(\mathcal{F})$. \cite{janzing2020feature} explicitly use the output of the model, called $\hat{Y}$ in the following, for their arguments and their algorithm. The approach presented in this work is also usable for $\hat{Y}$ instead of $Y$. 
A simplified, pragmatic adaption might be to treat $\hat{Y}$ as an estimate of $Y$, which is likely to yield comparable results for models with high accuracy (or low MSE in the case of regression). A more causal approach would be to add a node $\hat{Y}$ to the causal graph and draw arrows from each feature used for training to $\hat{Y}$ with the corresponding assignment function $f_{\hat{Y}}$ now just given by the learned machine learning model. Doing so will however not cure the problems inherent to \eqref{eq:shapley_values} that we outlined in the main text: Consider again the example with the causal diagram \ref{fig:causal_graph_diabetes_risk}. First, note that $Y$ can actually not simply be removed from the diagram as it acts as a mediator (a chain on a causal path) between $B$ and the nodes $G$ and $H$. Second, there is a collider bias between $H$ and $B$ when conditioning on $G$ so that even with $Y$ completely replaced by $\hat{Y}$ we would have terms in \eqref{eq:shapley_values} that are influenced by collider bias. 

Summarizing, we believe that the model-agnostic approach presented in this work is the most suitable and expressive to highlight feature relevance, as it uses the actual process underlying the data and simplifies matters by ignoring model specific aspects.

\paragraph{On the combinatorial coefficients $\gamma(\mathcal{S})$.} We mentioned above that there are various versions of Shapley values. These versions only differ in their choice of $v$ in \eqref{eq:shapley_values} but keep the same coefficients $\gamma(\mathcal{S})$. We here shortly recall the meaning of $\gamma(\mathcal{S})$ \citep{osborne1994course} to motivate why we also use them in the cc-Shapley values from Definition \ref{def:cc_Shapley}:

If we do not want to make any additional assumptions, it is unclear which features $\mathcal{S}\subseteq \mathcal{F}\backslash \{X_j\}$ to use as a context to judge the relevance of $X_j$. To decide this in a fair manner, imagine that we base our decision on a randomized experiment: We shuffle the order of features $\mathcal{F}=\{X_1,\ldots,X_n\}$ by uniformly drawing from all permutations on $\mathcal{F}$. We then use all variables that come before $X_j$ in such a ordering as context. The probability that the set $\mathcal{S}\subseteq \mathcal{F}$ is chosen as context in such an experiment is then just
\begin{align}
\label{eq:appendix_gamma_S}
\gamma(\mathcal{S})= \frac{|\mathcal{S}|! (|\mathcal{F}|-|\mathcal{S}|-1)!}{|\mathcal{F}|!} \,.
\end{align}
Here, $|\mathcal{S}|!$ are the number of ways we can reorder $\mathcal{S}$ in front of $X_j$, $(|\mathcal{F}|-|\mathcal{S}|-1)!$ are the number of ways we can order the remaining features after $X_j$, and the denominator $|\mathcal{F}|!$ is just the number of permutations in $\mathcal{F}$. In particular, we see that the coefficients in \eqref{eq:appendix_shapley_values_game_theory} add up to 1.

The same logic can be applied to motivate the usage of $\gamma(\mathcal{S})$ for the cc-Shapley values in Definition \ref{def:cc_Shapley}: We choose a random subset $\mathcal{S}\subseteq \mathcal{F}\backslash\{X_j\}$ as causal context using the same randomized procedure.

\paragraph{On the asymmetry of cc-Shapley values.} In Definition \ref{def:cc_Shapley} we introduce the cc-Shapley values as
\begin{align}
    \label{eq:shapley_values_do_appendix}
    \phi_{cc}(X_j) = \!\!\!\!\sum_{\mathcal{S}\subseteq \mathcal{F}\backslash\{X_j\}} \!\!\!\! \gamma(\mathcal{S})\, (\EE[Y|X_j,\doop(\mathcal{S})]- \EE[Y|\doop(\mathcal{S})])\,.
\end{align}
We motivated above the usage of the same coefficients $\gamma(\mathcal{S})$. However, in one other fundamental aspect this definition differs from \eqref{eq:appendix_shapley_values_game_theory}: The objects $\EE[Y|X_j,\doop(\mathcal{S})]$ and $\EE[Y|\doop(\mathcal{S})]$ cannot be understood as set functions as they treat $X_j$ and $\mathcal{S}$ differently. Our cc-Shapley values therefore lack the symmetry in the treatment of all considered features and, in general, will not have a property such as $\sum_{j} \phi(X_j)=v(\mathcal{F})-v(\emptyset)$ satisfied by \eqref{eq:appendix_shapley_values_game_theory} \citep{shapley1953value}. From a game theoretical viewpoint our asymmetric modification means that we give the ``player'' $X_j$ an advantage that we don't grant the other ``players'' $\mathcal{S}$, namely to use non-causal association. This asymmetric treatment is vital to the performance of $\phi_{cc}$. In fact,

\begin{itemize}
    \item We need to allow for non-causal association between $X_j$ and $Y$ to obtain non-trivial importance in anti-causal setups, cf. Section \ref{sec:methodology}.
    \item The context variables $\mathcal{S}$ cannot be treated in the same way, as we would otherwise allow $X_j$ to be attributed importance that solely arises from suppressing noise in the variables $\mathcal{S}$ as we observe for Example \ref{ex:diabetes_breakfast}.
\end{itemize}
Breaking the symmetric, equal treatment of all features in \eqref{eq:appendix_shapley_values_game_theory} is therefore precisely what allows \eqref{eq:shapley_values_do_appendix} to correct spurious associations.

\section{Detailed comparison with other methods}
\label{app:detailed_comparision_with_other_methods}
We here compare cc-Shapley with different methods from the literature that also consider Shapley values in the light of the causal structure of the data. Most of the methods listed below are not designed for an \emph{anti-causal} setting, that is a setting where $Y$ is a cause and not the effect of its features. Suppression as a phenomenon related to collider bias is, however, more naturally linked to such a setting. For the anti-causal setting of Example \ref{ex:diabetes_breakfast}, we find that the methods listed below do either not attribute relevance to the informative feature $G$ or also attribute relevance to the suppressor $C$. The only exception to this are (a modification of) asymmetric Shapley values (Section \ref{subsec:asymmetric_shapley_values}), where we however observe inconsistent behavior between causal and anti-causal settings. For better readability and comparability, we will rewrite all methods in a similar notation to ours.

\subsection{$\doop$-Shapley}
\label{subsec:do_Shapley}
\cite{jung2022measuring} propose \emph{$\doop$-Shapley} as a method for measuring causal contributions of each variable, which they define as
\begin{align}
\label{eq:do_Shapley}
    \phi_{\doop}(X_j) = \sum_{\mathcal{S}\subseteq \mathcal{F}\backslash\{X_j\}} {\small \frac{|\mathcal{S}|! (|\mathcal{F}|-|\mathcal{S}|-1)!}{|\mathcal{F}|!}} (\EE[Y|\doop(X_j,\mathcal{S})]-\EE[Y|\doop(X_j)]\,.
\end{align}
Additional aspects regarding the computation of do-Shapley are provided by \cite{parafita2025practical} and \cite{teal2026exactly}.
\cite{jung2022measuring} explicitly require $Y$ to be the last element in the topological order of the graph, whereas \cite{parafita2025practical} simply discard all variables that are not ancestors of $Y$. In fact, since \eqref{eq:do_Shapley} accounts only for causal contributions every non-causal association is discarded. Accordingly, we obtain for Example \ref{ex:diabetes_breakfast} 
\begin{align*}
    \phi_{\doop}(C)=\phi_{\doop}(G)=0\,.
\end{align*}

\subsection{Causal Shapley values}
\label{subsec:causal_Shapley_values}
\begin{figure}[t]
    \centering
    \includegraphics[width=0.5\linewidth]{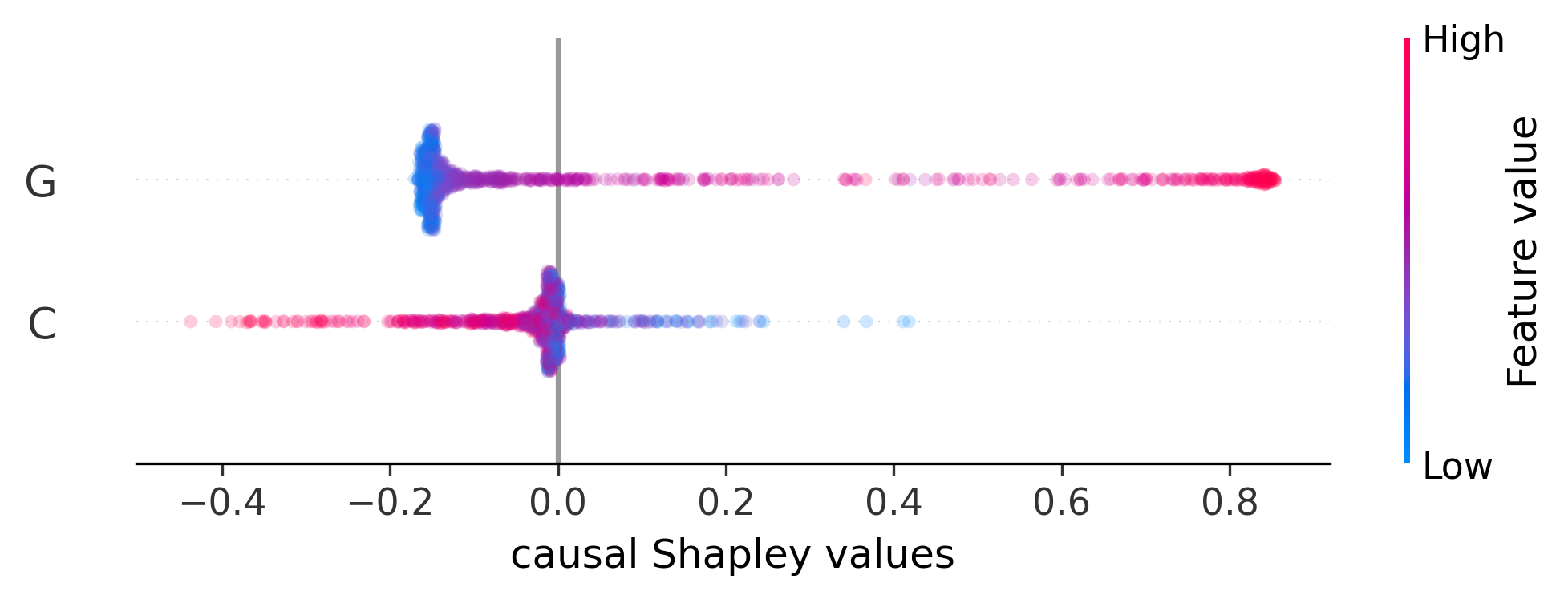}
    \caption{Causal Shapley values as proposed by \cite{heskes2020causal} evaluated for the running Example \ref{ex:diabetes_breakfast}.}
    \label{fig:causal_shapley_diabetes_breakfast}
\end{figure}

\cite{heskes2020causal} propose \emph{causal Shapley values} which are given as 
\begin{align}
\label{eq:causal_Shapley}
    \phi_{\mathrm{causal}}(X_j) = \sum_{\mathcal{S}\subseteq \mathcal{F}\backslash\{X_j\}} {\small \frac{|\mathcal{S}|! (|\mathcal{F}|-|\mathcal{S}|-1)!}{|\mathcal{F}|!}} (\EE[f(\mathcal{F})|\doop(X_j,\mathcal{S})]-\EE[f(\mathcal{F})|\doop(X_j)]\,,
\end{align}
with $f(\mathcal{F})$ being the machine learning model under consideration that uses all features $\mathcal{F}=\{X_1,\ldots,X_n\}$. The only difference to do-Shapley from Section \ref{subsec:do_Shapley} is thus the replacement of $Y$ by $f(\mathcal{F})$. For the anti-causal setup of Example \ref{ex:diabetes_breakfast} this creates, however, a fundamental difference in behavior, as $f(\mathcal{F})$ is the effect of all features $\mathcal{F}=\{C,G\}$. Figure \ref{fig:causal_shapley_diabetes_breakfast} shows the values of $\phi_{\mathrm{causal}}$ for $f(\mathcal{F})$ fitted as in Section \ref{app:details_on_diabetes_breakfast}. For the computation of $\phi_{\mathrm{causal}}$, we used the known SCM \eqref{eq:scm_diabetes_breakfast} and averaged $f(C,G)$ over 1000 samples from each intervened SCM to compute expectations. We observe that, similar to the behavior of conventional Shapley values, causal Shapley values falsely indicate a negative relevance for the suppressor $C$.

\subsection{Intrinsic causal contribution}
\cite{janzing2024quantifying} introduce the \emph{intrinsic causal contribution} (ICC) of a feature $X_j$ via 
\begin{align}
\label{eq:icc}
    \phi_{\mathrm{ICC}}(X_j) = \sum_{\mathcal{S}\subseteq \mathcal{V}\backslash\{X_j\}} {\small \frac{|\mathcal{S}|! (|\mathcal{V}|-|\mathcal{S}|-1)!}{|\mathcal{V}|!}} (\psi(Y|U_j,U_{\mathcal{S}})-\psi(Y|U_{\mathcal{S}}))\,,
\end{align}
where $\psi(Y|U_{\mathcal{S}})$ either denotes $\EE[\Var(Y|U_{\mathcal{S}})]$ (for regression) or the conditional entropy $H(Y|U_{\mathcal{S}})$ (for classification). Note that, in contrast to all other methods discussed in this section, \eqref{eq:icc} really sums over subsets of $\mathcal{V}=\mathcal{F}\cup \{Y\}$, i.e. subsets that might include $Y$. For Example \ref{ex:diabetes_breakfast}, we have for any $X\in \{C,G\}$ and $U_T\subseteq \{U_C,U_G,U_Y\}\backslash\{U_X\}$ that
\begin{align}
    \psi(Y|U_T)-\psi(Y|U_T,U_X) = 0 \,,
\end{align}
since $Y\ind U_X | U_T$. Thus, similar as for do-Shapley from Section \ref{subsec:do_Shapley}, we find that the ICC framework does not assign importance to any feature in Example \ref{ex:diabetes_breakfast}.

\subsection{Asymmetric Shapley values}
\label{subsec:asymmetric_shapley_values}
\cite{frye2020asymmetric} introduce \emph{asymmetric Shapley values}, which they define as
\begin{align}
    \label{eq:asymmetric_Shapley}
    \phi_{\mathrm{asym.}}(X_j)= \sum_{\pi} w_{\pi} (\EE[f(\mathcal{F})|\{X_i| \pi(i)\leq \pi(j)  \}) - \EE[f(\mathcal{F})|\{X_i| \pi(i) < \pi(j)  \}])\,.
\end{align}
Here the sum runs over all permutations $\pi$ on $\{1,\ldots,n\}$ and the $w_\pi\geq 0$ add up to 1. \cite{frye2020asymmetric} suggest to restrict $w_\pi\propto 1$ to only those permutations that are coherent with the topological ordering of the graph.
In the form \eqref{eq:asymmetric_Shapley} discussed by \cite{frye2020asymmetric}, asymmetric Shapley values are designed to work in a causal setting as $f(\mathcal{F})$ can be considered as the effect of all features, cf. Section \ref{subsec:causal_Shapley_values}. 
However, if we assume that $f(\mathcal{F})$ is the optimal model $\EE[Y|\mathcal{F}]$ (cf. Section \ref{app:conditional_expecation_as_optimal_model}), we can actually rephrase \eqref{eq:asymmetric_Shapley} as 
\begin{align}
    \label{eq:asymmetric_Shapley_modified}
    \phi_{\mathrm{asym.}}(X_j)= \sum_{\pi} w_{\pi} (\EE[Y|\{X_i| \pi(i)\leq \pi(j)  \}) - \EE[Y|\{X_i| \pi(i) < \pi(j)  \}])\,,
\end{align}
where we used the tower property of conditional expectations to replace $f(\mathcal{F})=\EE[Y|\mathcal{F}]$ with $Y$ within the conditional expectations of \eqref{eq:asymmetric_Shapley}. In the following, we discuss the behavior of \eqref{eq:asymmetric_Shapley_modified} in the presence of  suppression and in anti-causal settings.

For Example \ref{ex:diabetes_breakfast} there is only one permutation $\pi$ that respects the topological ordering of $\mathcal{F}=\{G,C\}$ and which yields
\begin{align}
    \label{eq:asymmetric_Shapley_diabetes_breakfast}
    \begin{aligned}
    \phi_{\mathrm{asym.}} (C) &= \EE[Y|C] -\EE[Y] = 0\,, \\
    \phi_{\mathrm{asym.}} (G) &= \EE[Y|G,C] - \EE[Y|C]= I_C(G) \neq 0  \,.
    \end{aligned}
\end{align}
Thus, asymmetric Shapley values \eqref{eq:asymmetric_Shapley} correctly indicate relevance of $G$ and no relevance for $C$. As can be seen in \eqref{eq:asymmetric_Shapley_diabetes_breakfast}, they can however enforce the usage of context: $\phi_{\mathrm{asym.}}(G)$ has no term comparable to $I_{\emptyset}$ from \eqref{eq:shapley_values}. This can lead to an inconsistent behavior in anti-causal settings when compared with a causal setting. To see this, consider the two simple cases illustrated in Figure \ref{fig:causal_anticausal_mediation}. Figure \ref{fig:causal_mediation} show a scenario where $X$ influences $Y$ through a mediator $M$. The asymmetric Shapley value for $X$ is given as 
\begin{align}
    \phi_{\mathrm{asym.}}(X)=\EE[Y|X]-\EE[Y]
\end{align}
which will, in general, be different from zero and thus $X$ is marked as relevant. When using the anti-causal meditation scenario in Figure \ref{fig:anticausal_mediation}, however, we obtain
\begin{align}
    \phi_{\mathrm{asym.}}(X)=\EE[Y|X,M] - \EE[Y|M] = 0\,,
\end{align}
since $X\ind Y |M$. Once more, as for Example \ref{ex:diabetes_breakfast}, we are forced to use context which yields in this case zero relevance of $X$. In other words, the impact of mediation on the relevance of a feature fundamentally differs between causal and anti-causal settings.

In summary, asymmetric Shapley values in the modified form \eqref{eq:asymmetric_Shapley_modified} are the only method discussed in this section which is well-behaved on Example \ref{ex:diabetes_breakfast}. However, in contrast to their behavior in a causal setting, and to the cc-Shapley values from Definition \ref{def:cc_Shapley}(which always possess a no-context term $I_{\emptyset}(X)$) , they are sensitive to mediation in an anti-causal setting.
\begin{figure}[t]
  \begin{subfigure}[t]{0.49\textwidth}
    \centering
    \begin{tikzpicture}[->, node distance=3cm, every node/.style={circle, draw}]
  \node (X) at (0,0) {$X$};
  \node (M) at (2.5,0) {$M$};
  \node (Y) at (5,0) {$Y$};
  \begin{scope}[->]
    \draw (X) to (M);
    \draw (M) to (Y);
  \end{scope}
\end{tikzpicture}
  \caption{Causal mediation}
  \label{fig:causal_mediation}
  \end{subfigure}
  \begin{subfigure}[t]{0.49\textwidth}
    \centering
    \begin{tikzpicture}[->, node distance=3cm, every node/.style={circle, draw}]
  \node (X) at (0,0) {$X$};
  \node (M) at (2.5,0) {$M$};
  \node (Y) at (5,0) {$Y$};
  \begin{scope}[->]
    \draw (Y) to (M);
    \draw (M) to (X);
  \end{scope}
\end{tikzpicture}
  \caption{Anti-causal mediation}
  \label{fig:anticausal_mediation}
  \end{subfigure}
  \caption{Simple cases of causal (left) and anti-causal (right) mediation between $X$ and $Y$ through $M$ on which asymmetric Shapley values show inconsistent behavior.}
  \label{fig:causal_anticausal_mediation}
\end{figure}

\section{Further theoretical aspects}
\subsection{Conditional expectation as optimal model}
\label{app:conditional_expecation_as_optimal_model}
The following basic result  recalls that the mathematically optimal model for a typical (supervised) machine learning problem is given by the conditional expectation. In particular, this shows that we can obtain an estimate of conditional expectations such as $\EE[Y|\mathcal{S}]$  with features $\mathcal{S}\subseteq \mathcal{F}$ simply by training a data-driven model with input $\mathcal{S}$ and target $Y$ using MSE-loss (for regression) or cross-entropy loss (for classification).

\begin{lemma}[Conditional expectation as optimally trained model] \label{lem:conditional_expectation_as_optimal_model} Consider a set of square-integrable random variables $\mathcal{S}$ and another square integrable and real-valued random variable $Y$ defined on the same probability space. Then the optimization problem
    \begin{align}
    \label{eq:optimal_l2_solution}
    \argmin_{f} \EE[|Y-f(\mathcal{S})|^2]
    \end{align}
    has (almost surely) a unique solution given by the conditional expectation $\EE[Y|\mathcal{S}]$. If $Y$ is binary, then $\EE[Y|\mathcal{S}]=P(Y=1|\mathcal{S})$ is moreover the optimal solution to the problem
    \begin{align}
    \label{eq:optimal_cross_entropy_solution}
    \argmin_{f} \EE[\mathcal{L}_{\mathrm{CE}}(Y;f(\mathcal{S}))] \,,
    \end{align}
    where $\mathcal{L}_{\mathrm{CE}}(Y;f(S))=-Y\log f(\mathcal{S}) - (1-Y) (1-\log f(\mathcal{S}))$ is the cross-entropy loss.
\end{lemma}

\begin{remark}
A similar statement holds for the multivariate case or multiclass case (when $Y$ is one-hot-encoded).
\end{remark}

\begin{proof}
The statement on $L2$-optimality \eqref{eq:optimal_l2_solution} is a classical result from probability theory \citep{loeve2013probability}. The statement based on \eqref{eq:optimal_cross_entropy_solution} can be seen by using the tower property of conditional expectations:
\begin{align*}
    &\EE[\mathcal{L}_{\mathrm{CE}}(Y;f(\mathcal{S}))] = \EE[\EE[\mathcal{L}_{\mathrm{CE}}(Y;f(\mathcal{S}))|\mathcal{S}]] \\ &= -\EE[p_\mathcal{S}\log f(\mathcal{S}) + (1-p_\mathcal{S})(1-\log f(\mathcal{S}))] \,,
\end{align*}
where $p_\mathcal{S}=P(Y=1|\mathcal{S})$.
Minimizing the inner expression w.r.t $f(\mathcal{S})$ (setting the derivative to zero), yields $f(\mathcal{S})=p_{\mathcal{S}}=P(Y=1|\mathcal{S})=\EE[Y|\mathcal{S}]$.
\end{proof}

\subsection{Proofs of Lemmas \ref{lem:irrelevant_context} and \ref{lem:observation_equals_intervention}}
\label{app:proofs_of_lemmas}
We here restate the lemmas from Section \ref{subsec:estimation} together with their proofs. 

\begin{lemma}[Irrelevant context]
Consider $X_j\in \mathcal{F}, \mathcal{S}\subseteq \mathcal{F}\backslash\{X_j\}$ and assume that there are no causal paths $\mathcal{S} \ipath Y, X_j$. Then we have $\EE[Y|X_j, \doop(\mathcal{S})]=\EE[Y|X_j]$ and $I_{\doop(\mathcal{S})}(X_j)=I_{\emptyset}(X_j)$.  
\end{lemma}
\begin{proof}
We apply rule 3 of $\doop$ calculus \citep{pearl2009causality} to remove $\doop(\mathcal{S})$ from the conditional expectation, i.e., we have to verify the following d-separation
\begin{align}
    \label{eq:rule_3_do_calculus}
    Y \ind_{\mathcal{G}_{\overline{\mathcal{S}}}} \mathcal{S}\,\, |  \,\,  X_j  \,,
\end{align}
where $\mathcal{G}_{\overline{\mathcal{S}}}$ denotes the graph where all incoming edges to $\mathcal{S}$ were deleted.
The actual rule 3 uses $\mathcal{G}_{\overline{\mathcal{S}(X_j)}}$ instead of $\mathcal{G}_{\overline{\mathcal{S}}}$, which denotes the removal of all ancestors of $X_j$ from $\mathcal{S}$. By assumption, $\mathcal{S}$ contains, however, no ancestors of $X_j$. Consider, conditional on $X_j$, a potential unblocked path between $\mathcal{S}$ and $Y$ in $\mathcal{G}_{\overline{\mathcal{S}}}$. Denote by $X_k\in \mathcal{S}$ the last element within $\mathcal{S}$ on the way to $Y$ and consider, from now on, the unblocked sub-path from $X_k$ to $Y$. As we removed ingoing edges to $\mathcal{S}$, this path must have an outgoing edge at $X_k$. By assumption, there is no path $X_k \ipath Y$ . The only remaining path with an outgoing edge at $X_k$ in $\mathcal{G}_{\overline{\mathcal{S}}}$ that is unblocked conditional on $X_j$ can be a path on which $X_j$ or its ancestors $\ancestors(X_j)$ serve as colliders and there are no other colliders on this path. Let $C\in \ancestors(X_j)\cup \{X_j\}$ denote the first of these colliders. We know that there is no other collider between $X_k$ and $C$ and that the path leaves $X_k$ with outgoing edge, we must thus have $X_k \ipath C$ (the arrows cannot flip direction in between). However, this cannot be true as $C\in \ancestors(X_j)\cup \{X_j\}$ would imply that there is a path $X_k \ipath X_j$ which violates our assumption. This shows \eqref{eq:rule_3_do_calculus} and hence $\EE[Y|X_j, \doop(\mathcal{S})]=\EE[Y|X_j]$. 

As we assumed that there are no causal paths $\mathcal{S} \ipath Y$ we have further $\EE[Y|\doop(\mathcal{S})]=\EE[Y]$ from which the identity $I_{\doop(X_k)}(X_j)=I_{\emptyset}(X_j)$ follows.
\end{proof}

\begin{lemma}[Intervention equals observation]
Consider $X_j\in \mathcal{F},\,\mathcal{S}\subseteq \mathcal{F}\backslash\{X_j\}$ such that, either
\begin{itemize}
    \item \emph{(no backdoor paths)} there are no unblocked backdoor paths from $\mathcal{S}$ to $X_j$ or from $\mathcal{S}$ to $Y$, or
    \item \emph{(purely causal setup)} there are no causal paths $Y \ipath X_j,\mathcal{S}$ and no confounders $H \in \mathcal{V}\backslash \left(\{Y,X_j\}\cup\mathcal{S}\right)$ with $X_j \pathi H \ipath Y$ or $\mathcal{S} \pathi H \ipath Y$,
\end{itemize}

then we have the identities $\EE[Y|X_j,\doop(\mathcal{S})] = \EE[Y|X_j,\mathcal{S}]$ and $I_{\doop(\mathcal{S})}(X_j)=I_{\mathcal{S}}(X_j)$.
\end{lemma}

\begin{proof}
We will show the identity for the conditional expectation first. For both scenarios, "no backdoor paths" and "purely causal setup"  we will use rule 2 of $\doop$ calculus \citep{pearl2009causality} to replace the $\doop$-operation by conditioning. We have to show
\begin{align}
    \label{eq:rule_2}
    Y \ind_{\mathcal{G}_{\underline{\mathcal{S}}}} \mathcal{S} | X_j  \,,
\end{align}
where $\mathcal{G}_{\underline{\mathcal{S}}}$ denotes the graph where all outgoing edges from $\mathcal{S}$ were removed. Consider first the ``no backdoor paths'' scenario:
Any path between $\mathcal{S}$ and $Y$ in $\mathcal{G}_{\underline{\mathcal{S}}}$ that is unblocked conditional on $X_j$ will have an incoming edge when leaving $\mathcal{S}$ (as it runs in $\mathcal{G}_{\underline{\mathcal{S}}}$) and either 1) not contain $X_j$ or 2) contain $X_j$ or ancestors of $X_j$ as colliders and no other colliders. In the first case, we have just an unblocked backdoor path from $\mathcal{S}$ to $Y$ and in the second case we can subtract an unblocked subpath which is a backdoor path between $\mathcal{S}$ and $X_j$. Both violate our assumption and therefore we conclude \eqref{eq:rule_2}.

For the "purely causal setup" scenario consider again a path between $\mathcal{S}$ and $Y$ in $\mathcal{G}_{\underline{\mathcal{S}}}$ that is unblocked conditional on $X_j$. Denote by $X_k\in \mathcal{S}$ the last element of this path within $\mathcal{S}$ on the way to $Y$ and consider from now on the unblocked subpath between $X_k$ and $Y$. Since we removed outgoing edges from $X_k$ this path can only have an incoming edge to $X_k$. Due to our assumptions this path cannot be of the form $X_k \pathi Y$. The only remaining paths that are unblocked, conditional on $X_j$, have a fork. This fork cannot be $\leftarrow X_j \rightarrow$ as it would be blocked conditional on $X_j$. We also know, by assumption, that there is no $H$ such that $X_k \pathi H \ipath Y$. The arrows must thus flip directions before or after the fork, which means that there is a collider.  The only way this cannot lead to a blocked path is that $X_j$ or a set of ancestors of $X_j$ are these colliders and there are no other colliders. We will write $\mathcal{C}\subseteq \{X_j\} \cup \ancestors(X_j)$ for this set. Call $C \in \mathcal{C}$ the last element in $\mathcal{C}$ on the way of the considered path to $Y$. Since $C$ is a collider, as there is no further collider between $C$ and $Y$ on the path  and as the remaining bit of the path between $C$ and $Y$ is unblocked we must either have  $C \pathi H \ipath Y$ for some $H$ or $C \pathi Y $, depending on the direction of the last edge of the path. In either case we obtain due to $C\in \{X_j\} \cup \ancestors(X_j)$ that either $X_j \pathi H \ipath Y$ or $X_j \pathi Y$, which both violate our assumptions. Hence, there is no unblocked path and \eqref{eq:rule_2} follows.

We still have to show $\EE[Y|\doop(\mathcal{S})]=\EE[Y|\mathcal{S}]$, which follows again from rule 2 of $\doop$ calculus once we know 
\begin{align*}
Y \ind_{\mathcal{G}_{\underline{\mathcal{S}}}} \mathcal{S}\,.
\end{align*}
For the "no backdoor paths" scenario this follows from the assumption that there are no backdoor paths from $\mathcal{S}$ to $Y$. For the "purely causal setup" scenario we can use that any unconditional unblocked backdoor path is either form $\mathcal{S} \pathi Y$ or $\mathcal{S} \pathi H \ipath Y$, which are cases which we both excluded in our assumption.
\end{proof}

\subsection{Backdoor adjustment}

The following lemma provides an adaption of the ``backdoor adjustment'' \citep[Theorem 3.3.2]{pearl2009causality} for the context of this work.

\begin{lemma}[Version of backdoor adjustment]
\label{lem:backdoor_adjustment}
Consider $X_j\in \mathcal{F},\,\mathcal{S},\mathcal{W}\subseteq \mathcal{F}\backslash\{X_j\}, \mathcal{S}\cap \mathcal{W}=\emptyset$ such that
\begin{itemize}
    \item There are no causal paths $\mathcal{S} \ipath \mathcal{W}\cup\{X_j\}$
    \item Both, $\mathcal{W}$ and $\mathcal{W}\cup\{X_j\}$, block all backdoor paths between $\mathcal{S}$ and $Y$.
\end{itemize}
then we have the identity
\begin{align*}
   I_{\doop(\mathcal{S})} (X_j)=\EE_{\mathcal{W}|X_j}[\EE[Y|X_j,\mathcal{S},\mathcal{W}]] - \EE_{\mathcal{W}}[\EE[Y|\mathcal{S},\mathcal{W}]] \,.
\end{align*}
\end{lemma}

\begin{proof}
Recall that $I_{\doop(\mathcal{S})} (X_j)=\EE[Y|X_j,\doop(\mathcal{S})]-\EE[Y|\doop(\mathcal{S})]$. By the rules for conditional expectations, we have
\begin{align*}
    \EE[Y|X_j,\doop(\mathcal{S})] &= \EE_{\mathcal{W}|X_j,\doop(\mathcal{S})} [\EE[Y|X_j,\doop(\mathcal{S}),\mathcal{W}]] \mbox{ and } \\
    \EE[Y|\doop(\mathcal{S})] &= \EE_{\mathcal{W}|\doop(\mathcal{S})} [\EE[Y|\doop(\mathcal{S}),\mathcal{W}]]\,.
\end{align*}
Since the sets $\mathcal{W}$ and $\mathcal{W}\cup\{X_j\}$ block all backdoor paths between $\mathcal{S}$ and $Y$ we can replace the inner $\doop$ operation by conditioning (rule 2 of $\doop$ calculus \citep{pearl2009causality}) and obtain
\begin{align*}
    \EE[Y|X_j,\doop(\mathcal{S})] &= \EE_{\mathcal{W}|X_j,\doop(\mathcal{S})} [\EE[Y|X_j,\mathcal{S},\mathcal{W}]] \mbox{ and } \\
    \EE[Y|\doop(\mathcal{S})] &= \EE_{\mathcal{W}|\doop(\mathcal{S})} [\EE[Y|\mathcal{S},\mathcal{W}]]\,.
\end{align*}
To conclude, we remove the $\doop$ intervention in the outer expectation using $P(X_j,\mathcal{W}|\doop(\mathcal{S}))=P(X_j,\mathcal{W})$, which follows from the fact that there are no causal paths $\mathcal{S} \ipath \mathcal{W}\cup\{X_j\}$ (rule 3 of $\doop$ calculus \citep{pearl2009causality}).
\end{proof}

Application of Lemma \ref{lem:backdoor_adjustment} requires specific conditions on the relation between $X_j$ and $\mathcal{S}$ as well as a search for a suitable set $\mathcal{W}$. Once found, we have to find a reasonable estimate for the conditional probability $P(\mathcal{W}|X_j)$, which can be a complicated task, especially for higher dimensional $\mathcal{W}$. In the experiments in Section \ref{sec:experimental_results}, we therefore found a it more practical to first estimate the assignment functions in the SCM followed by the application of Algorithm \ref{alg:compute_importance}. Identification via more advanced techniques as in \cite{teal2026exactly} or \cite{jung2022measuring} could be a feasible alternative but were not studied in the scope of this work.

\subsection{Reasoning behind Algorithm \ref{alg:compute_importance}}
\label{app:justification_of_algorithm}
Algorithm \ref{alg:compute_importance} implicitly claims that fitting data-driven models in the modified model $\mathcal{M}^{\doop(\mathcal{S} \sim q)}$ yields $\EE[Y|\doop(\mathcal{S})]$ and $\EE[Y|X_j, \doop(\mathcal{S})]$, where we denote by $q$ the joint marginal of $\mathcal{S}$ within the original SCM $\mathcal{M}$.

We already recalled in Lemma \ref{lem:conditional_expectation_as_optimal_model} the established fact that fitting a data driven model with standard L2-loss (regression) or cross-entropy-loss (classification) gives an estimate of the conditional expectation. We are thus left with the claim that
\begin{align}
\label{eq:algorithm_claim}
\begin{aligned}
    \EE[Y|X_j, \doop(\mathcal{S})] &= \EE_{\mathcal{M}^{\doop(\mathcal{S}\sim q)}}[Y|X_j,\mathcal{S}] \,, \\
    \EE[Y|\doop(\mathcal{S})] &= \EE_{\mathcal{M}^{\doop(\mathcal{S}\sim q)}}[Y|\mathcal{S}]
\end{aligned}
\end{align}
to motivate Algorithm \ref{alg:compute_importance}. The causal graph of $\mathcal{M}^{\doop(\mathcal{S}\sim q)}$ arises from the one of $\mathcal{M}$ by deleting all incoming edges to $\mathcal{S}$. In particular, in $\mathcal{M}^{\doop(\mathcal{S}\sim q)}$ there are no backdoor paths from $\mathcal{S}$ to $Y$, neither blocked nor unblocked. Hence rule 2 of $\doop$ calculus \citep{pearl2009causality} applies and we have
\begin{align}
\label{eq:remove_do_modified_scm}
\begin{aligned}
    \EE_{\mathcal{M}^{\doop(\mathcal{S}\sim q)}}[Y|X_j, \doop(\mathcal{S})] &= \EE_{\mathcal{M}^{\doop(\mathcal{S}\sim q)}}[Y|X_j,\mathcal{S}] \,, \\
    \EE_{\mathcal{M}^{\doop(\mathcal{S}\sim q)}}[Y|\doop(\mathcal{S})] &= \EE_{\mathcal{M}^{\doop(\mathcal{S}\sim q)}}[Y|\mathcal{S}] \,.
\end{aligned} 
\end{align}
Now, once we intervene on $\mathcal{S}$ via $\doop(\mathcal{S})$ to set it to a fixed value, the actual marginal of $q$ does not influence the distribution of all non-intervened variables $\mathcal{V}\backslash\mathcal{S}$ so that we can perform the transition $\mathcal{M}^{\doop(\mathcal{S})} \rightarrow \mathcal{M}$ in the following equations
\begin{align}
\label{eq:remove_modified_model}
\begin{aligned}
    \EE_{\mathcal{M}^{\doop(\mathcal{S}\sim q)}}[Y|X_j, \doop(\mathcal{S})] &= \EE[Y|X_j,\doop(\mathcal{S})] \,, \\
    \EE_{\mathcal{M}^{\doop(\mathcal{S}\sim q)}}[Y|\doop(\mathcal{S})] &= \EE[Y|\doop(\mathcal{S})] \,.
\end{aligned}
\end{align}

Combing \eqref{eq:remove_modified_model} with \eqref{eq:remove_do_modified_scm}, we obtain the identity \eqref{eq:algorithm_claim} that underpins Algorithm \ref{alg:compute_importance}.

\paragraph{The exact choice of $q$ doesn't matter.} Note, that arguments above never used the fact that $q$ is the marginal of $\mathcal{S}$  within the original model $\mathcal{M}$. In fact, we can take any distribution $q$ which is supported on all $\mathcal{S}$ values which we want to set as context. Another possible choice might be to combine the marginals of $X_k\in \mathcal{S}$ in an independent manner. For the experiments in Section \ref{sec:experimental_results} we found that such a proceeding gives indeed practically indistinguishable results to the ones presented here.

\section{Experimental details and additional results}

Throughout this work, with the exception of the linear SCM experiment, we used \texttt{XGBoost} \citep{Chen2016XGBoost} for fitting data-driven models.

\subsection{Diabetes and breakfast example (Example \ref{ex:diabetes_breakfast})}
\label{app:details_on_diabetes_breakfast}

Example \ref{ex:diabetes_breakfast} from Section \ref{sec:introduction} can be described by an SCM of the following form
\begin{align}
    \label{eq:scm_diabetes_breakfast}
    \begin{aligned}
    C &= U_C, \, U_C \sim \mathcal{N} (60;25^2)\,, \\
    Y &= U_Y, \, U_Y \sim \mathrm{Bernoulli}(0.15)\,, \\
    G &= 85 + 0.4 \cdot C + 40 \cdot Y + U_G, U_G \sim \mathcal{N}(0;10^2)\, ,
    \end{aligned}
\end{align}
where $\mathrm{Bernoulli}(p)$ denotes the Bernoulli distribution with success probability $p$ and $\mathcal{N}(\mu;\sigma^2)$ denotes the normal distribution with mean $\mu$ and variance $\sigma^2$. For computing the Shapley values \eqref{eq:shapley_values} displayed in Figure \ref{fig:shapley_values_diabetes_breakfast} (left), we fitted all conditional expectations via XGBoost on $3 \cdot 10^6$ random samples from \eqref{eq:scm_diabetes_breakfast}, using Lemma \ref{lem:conditional_expectation_as_optimal_model} above. For the cc-Shapley values \eqref{eq:shapley_values_do}, we applied Algorithm \ref{alg:compute_importance}: For each term in \eqref{eq:shapley_values_do}, we first modified the SCM accordingly and then used, once more, XGBoost on $3 \cdot 10^6$ random samples to fit each conditional expectation. Figure \ref{fig:shapley_values_diabetes_breakfast} was then produced by evaluating \eqref{eq:shapley_values} and \eqref{eq:shapley_values_do} on $10^4$ (new) data points drawn from \eqref{eq:scm_diabetes_breakfast}.

All $I$ objects involved in the computation of the Shapley values are plotted in Figure \ref{appfig:diabetes_breakfast_shape_functions_shapley}. For cc-Shapley the according objects are plotted in Figure \ref{appfig:diabetes_breakfast_shape_functions_do_shapley}. The univariate terms $I_{\emptyset}(C)$ and $I_{\emptyset}(G)$, shown on the diagonal in Figures \ref{appfig:diabetes_breakfast_shape_functions_shapley} and \ref{appfig:diabetes_breakfast_shape_functions_do_shapley}, agree for both approaches by definition.

\paragraph{Derivation of $I_{C}(G)$.}In Section \ref{subsec:using_the_causal_context} we gave an explicit expression for $I_{C}(G)$, which we derive here.
By definition we have, using the independence of $Y$ and $C$,
\begin{align*}
    I_{C}(G) &=\EE[Y|G,C]  - \EE[Y|C] \\ 
    &= P(Y=1|G,C) - P(Y=1|C) \\
    & = P(Y=1|G,C) - P(Y=1) \\
    & =  P(Y=1|G,C) - 0.15 \,.
\end{align*}
Thus, it remains to estimate $P(Y=1|G,C)$.
From the SCM \eqref{eq:scm_diabetes_breakfast} we obtain the joint distribution
\begin{align*}
    P(Y,G,C)=\mathrm{Bernoulli}(Y|0.15) \cdot  \mathcal{N}(C|60;25^2) \cdot \mathcal{N}(G|0.4C+85+40Y;10^2) \,.
\end{align*}
Writing $p=0.15$ we can use this expression to compute
\begin{align*}
    P(Y=1|G,C)&=\frac{P(Y=1,G,C)}{P(Y=1,G,C)+P(Y=0,G,C)} \\
        &= \frac{1}{1+\frac{1-p}{p}\frac{\mathcal{N}(G|0.4C+85;10^2)}{\mathcal{N}(G|0.4C+125;10^2)}} \\
        &= \frac{1}{1+\frac{1-p}{p}e^{-\frac{1}{2\cdot 10^2}[(G-0.4C-85)^2 - (G-0.4C-125)^2]}} \\
        &= \frac{1}{1+  e^{-[\frac{2}{5} (G-0.4C)-42-\log\frac{1-p}{p}]}} \\
        &= \sigma\left(\frac{2}{5}\left(G-0.4C - 105 - \frac{5}{2} \log \frac{17}{3}\right)\right) \\
        &\simeq \sigma\left(\frac{2}{5}(G-0.4C - 109)\right)\,,
\end{align*}
where $\sigma$ denotes the sigmoid function. In total, we thus have
\begin{align*}
    I_{C}(G)  \simeq \sigma\left(\frac{2}{5}(G-0.4C - 109)\right) - 0.15 \,.
\end{align*}

\begin{figure}[h!]
\centering
\setlength{\tabcolsep}{6pt}
\renewcommand{\arraystretch}{1.2}

\begin{tabular}{c C{0.33\textwidth} C{0.33\textwidth}}

 & $I_{\emptyset}(C)$
 & $I_{G}(C)$ \\

  & \hspace{-0.8cm}\includegraphics[width=\linewidth]{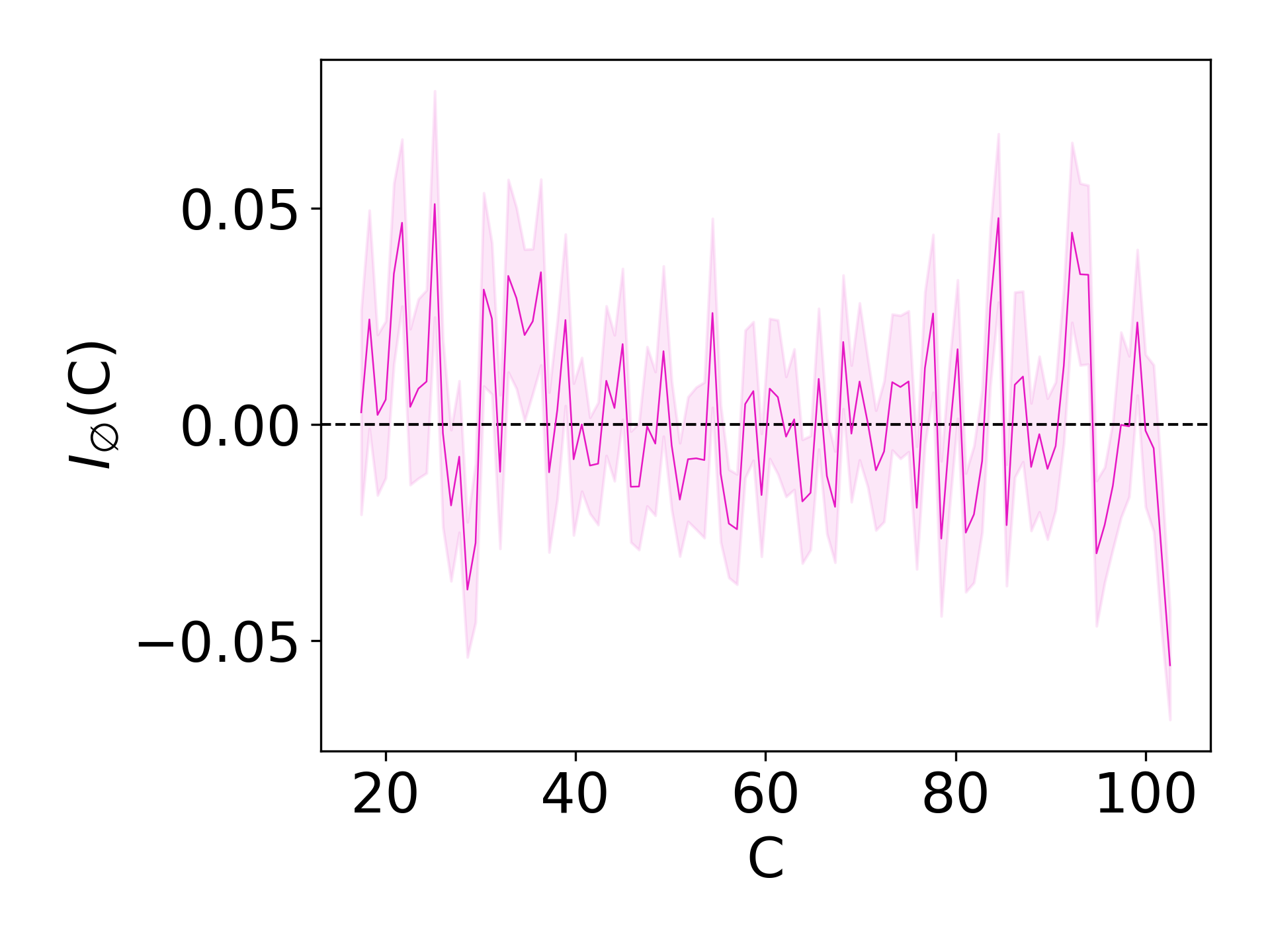}
 & \hspace{-0.4cm}\includegraphics[width=\linewidth]{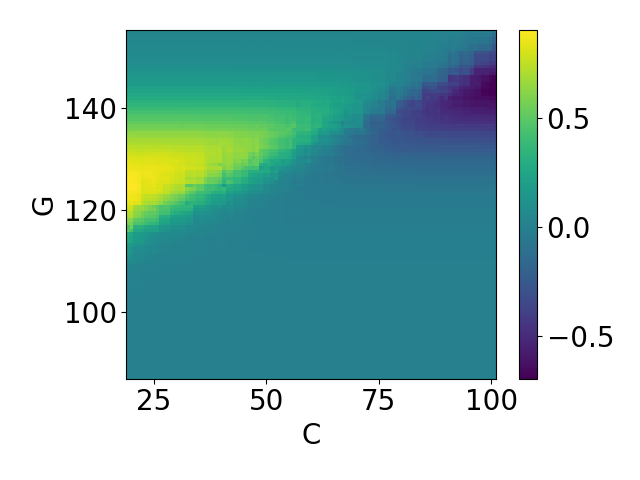}\\ 

 & $I_C(G)$
 & $I_\emptyset(G)$ \\

 & \includegraphics[width=\linewidth]{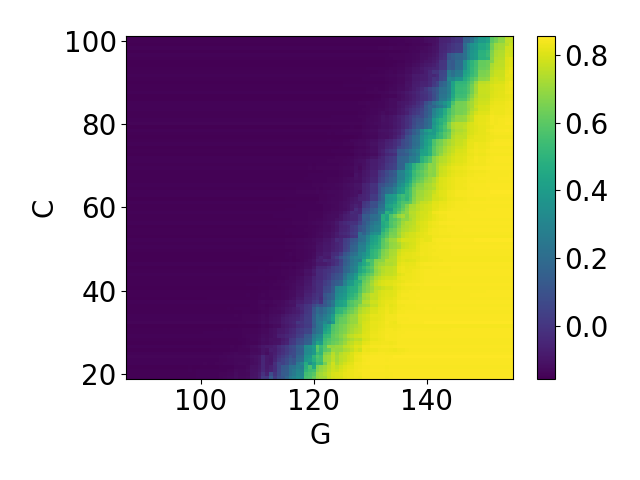}
 & \includegraphics[width=\linewidth]{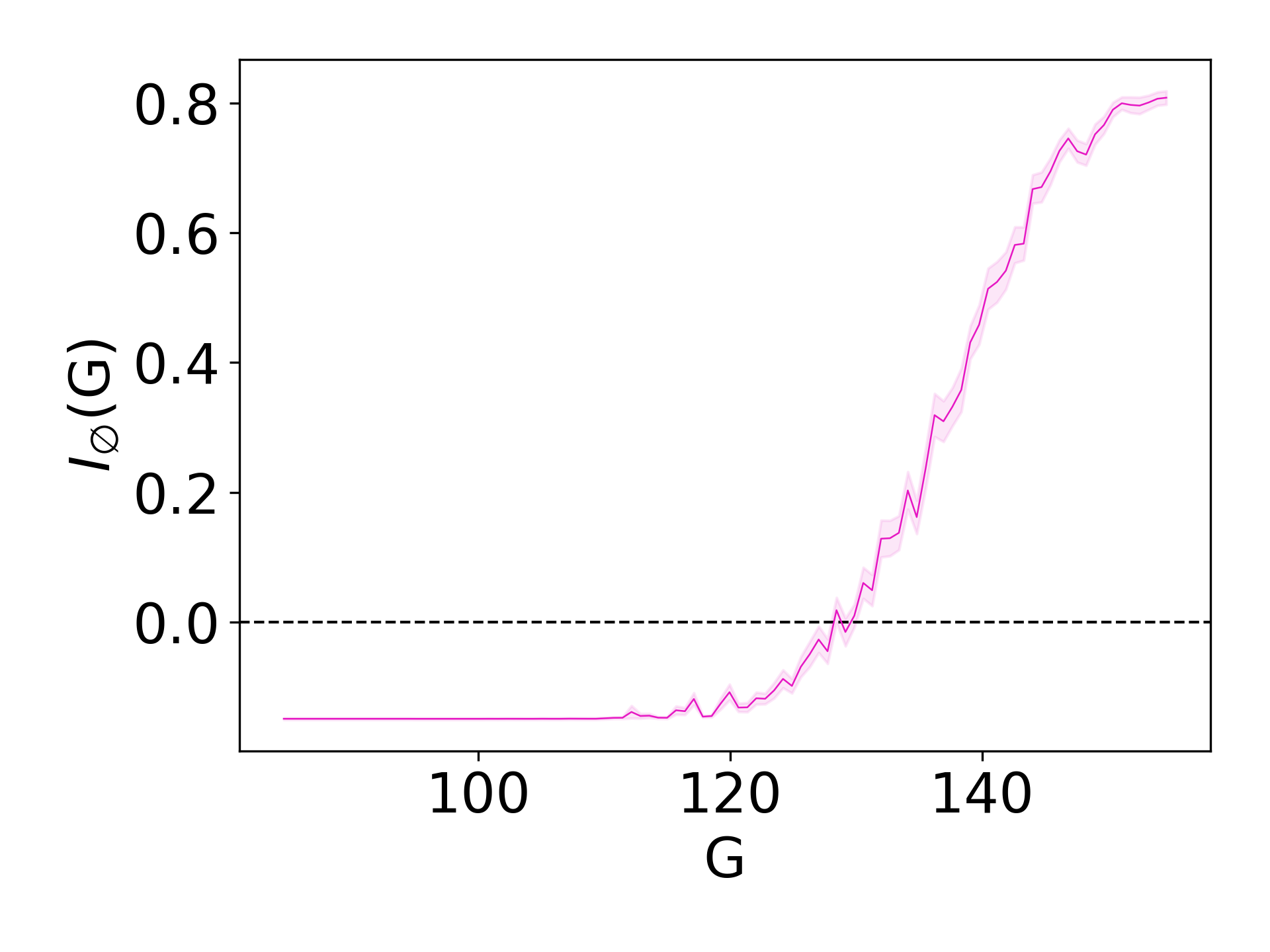}\\ 

\end{tabular}
\caption{Plot of all objects arising in the computation of the conventional Shapley values \eqref{eq:shapley_values} for Example \ref{ex:diabetes_breakfast}.}
\label{appfig:diabetes_breakfast_shape_functions_shapley}
\end{figure}

\begin{figure}[h!]
\centering
\setlength{\tabcolsep}{6pt}
\renewcommand{\arraystretch}{1.2}

\begin{tabular}{c C{0.33\textwidth} C{0.33\textwidth}}

 & $I_{\emptyset}(C)$
 & $I_{\doop(G)}(C)$ \\

 & \hspace{-0.8cm}\includegraphics[width=\linewidth]{figures/univariate_importance_diabetes_breakfast_Y_features_C.png}
 & \hspace{-0.4cm}\includegraphics[width=\linewidth]{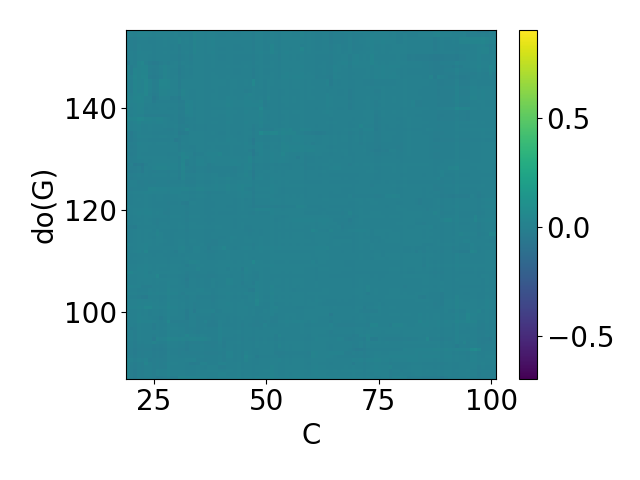}\\ 

 & $I_{\doop(C)}(G)$
 & $I_\emptyset(G)$ \\

 & \includegraphics[width=\linewidth]{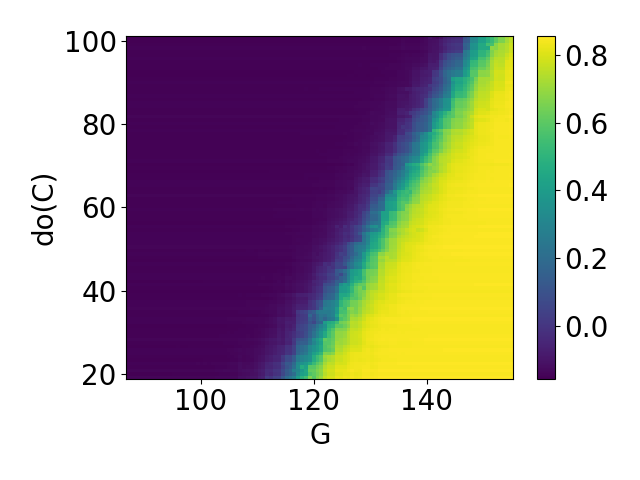}
 & \includegraphics[width=\linewidth]{figures/univariate_importance_diabetes_breakfast_Y_features_G.png}\\ 

\end{tabular}
\caption{Plot of all objects arising in the computation of the cc-Shapley values \eqref{eq:shapley_values_do} for Example \ref{ex:diabetes_breakfast}.}
\label{appfig:diabetes_breakfast_shape_functions_do_shapley}
\end{figure}
\newpage

\subsection{Linear SCM experiment}
\label{app:details_on_linear_SCM_experiment}
For the experiment underlying Figure \ref{fig:linear_scm_experiment}, we sampled 3,000 linear SCMs with 9 variables (including the target). To this end, we applied Algorithm \ref{alg:random_linear_SCM} with $n=9$ and $p=0.8$ to obtain a random adjacency matrix $A$ and then used i.i.d. random variables $U_1,U_2,\ldots,U_9 \sim \mathrm{Laplace}(0;0.1)$ (Laplace distribution with mean $0$ and scale $0.1$) to define the linear SCM
\begin{align}
    \label{eq:linear_SCM}
    \mathbf{V} = A \mathbf{V} + \mathbf{U} \,,
\end{align}
where $\mathbf{V}\in \RR^9$ denotes the variables of the SCM and $\mathbf{U}=(U_i)_{1\leq i\leq 9}$ collects the noise variables. For each noise instance $\mathbf{U}$, equation \eqref{eq:linear_SCM} was solved for $\mathbf{V}$ via $\mathbf{V}=(I-A)^{-1}\mathbf{U}$.
\begin{algorithm}[h]
\caption{Sample adjacency matrix of a linear SCM}
\label{alg:random_linear_SCM}
\KwIn{Number of variables $n$, edge probability $p\in [0,1]$}
\KwOut{adjacency matrix $A\in \RR^{n\times n}$}
\vspace{0.4em}  
\Comment{Create random binary matrix}
Sample i.i.d. $U_{i,j}\sim\mathrm{Uniform}(0,1)$ for ${1\leq i,j\leq n}$ \;
Create $A\leftarrow [U_{i,j} < p]_{1\leq i,j\leq n}$\;

\vspace{0.4em}  
\Comment{Turn into binary adjacency matrix}
For all $i\geq j$ set $A_{ij} \leftarrow 0$ \;
Draw a random $n$-permutation $\pi$ \;
Permute rows and columns of $A$ with $\pi$\;

\vspace{0.4em}
\Comment{Fill with random floats}
Sample i.i.d. $\alpha_{ij}\sim \mathcal{N}(0;1)$ for ${1\leq i,j\leq n}$\;
Replace $A_{ij}\leftarrow A_{ij} \cdot \alpha_{ij}$\;

\vspace{0.4em} 
\Return $A$
\end{algorithm}

From each SCM constructed as above, we sampled $3\cdot 10^4$ samples of $\mathbf{V}$. Target and features are constructed from $\mathbf{V}$ via the assignment:
\begin{align}
    Y &=V_1\,, \\
    X_i &= V_{i+1} \mbox{ for all } 1\leq i \leq 8.
\end{align}
The analysis behind Figure \ref{fig:linear_scm_experiment} then focuses on the relation between $Y,X_1$ and $X_2$.

\paragraph{Computation of $b_{X_1},\,b_{X_1|X_2}, \,b_{X_1,\doop(X_2)}$.}Since $I_{\emptyset}(X_1),I_{X_2}(X_1)$ and $I_{\doop(X_2)}(X_1)$ only differ from $\EE[Y|X_1],\EE[Y|X_1,X_2] $ and $\EE[Y|X_1,\doop(X_2)]$ via $X_1$-independent terms we used the latter to obtain
 $b_{X_1},\,b_{X_1|X_2}, \,b_{X_1,\doop(X_2)}$. For the observational quantities $b_{X_1},\,b_{X_1|X_2}$ we fitted linear univariate and bivariate models with least-squares loss on the $3\cdot 10^4$ samples described above. For the causal quantity $b_{X_1|\doop(X_2)}$, we made a stochastic intervention on $X_2$ in \eqref{eq:linear_SCM} as described in Algorithm \ref{alg:compute_importance} and then conducted another least-squares fit of a bivariate linear model on $3\cdot 10^4$ samples of the modified model. We used least-squares fitting, instead of maximum likelihood fitting, due to the optimality of conditional expectation under the $L2$-norm, cf. Lemma \ref{lem:conditional_expectation_as_optimal_model} above.

\paragraph{Collider impact.} The ``Collider impact'' indicated in the colorbar of Figure \ref{fig:linear_scm_experiment} is measured via a heuristic quantity that is defined as follows:

\begin{align}
    \label{eq:collider_impact}
    \mbox{Collider impact} = \frac{|CP_{X_2}|}{|CP_{X_2}|+|UP_{X_2}|} \,,
\end{align}
where
\begin{align*}
    CP_{X_2} &= \sum_{\mbox{\small $X_1,Y$ paths containing $X_2$ as the only collider}} \prod A_{ij}\,,\\
    UP_{X_2} &= \sum_{\mbox{\small unblocked $X_1,Y$ path containing $X_2$}}\prod A_{ij}
\end{align*}
denote the sum of products of entries of the adjacency matrix along the paths indicated under the summation symbols. In particular, if all paths between $X_1$ and $Y$ that run through $X_2$ and are either unblocked or contain $X_2$ as the only collider, all fall in the second category, the ``Collider impact'' will be 1. If they are all unblocked, ``Collider Impact'' will be 0.  Object \eqref{eq:collider_impact} is (loosely) motivated by Wright's path tracing rules \citep{wright1934method,pearl2013linear}.

For the analysis of paths between $X_1$ and $Y$, we used the \texttt{NetworkX} package from \cite{hagberg2007exploring}.

\subsection{Diabetes and BMI example}
\label{app:details_on_diabetes_risk}
The precise SCM for the nonlinear example outlined in Section \ref{sec:experimental_results} was entirely hand-crafted and is given by 
\begin{align}
\label{eq:scm_diabetes_risk}
\begin{aligned}
    B &= U_B,\, U_B\sim \mathcal{N}(25;5^2) \\
    Y &\sim \mathrm{Bernoulli}(\sigma(-2 + 0.1 \cdot (B-25))) \\
    H &= 5 + 10 \cdot Y  + 0.01 \cdot B^2 + U_H,\, U_H \sim \mathcal{N}(0;1) \\
    G &= 90 + 20 \cdot Y + 30 \cdot \sigma(-0.5 \cdot (H-5)) + B + U_G,\, U_G \sim \mathcal{N}(0;5^2) \;,
\end{aligned}
\end{align}
where $\sigma$ denotes the sigmoid function.
\begin{figure}[p]
\centering
\setlength{\tabcolsep}{6pt}
\renewcommand{\arraystretch}{1.2}

\begin{tabular}{c C{0.30\textwidth} C{0.30\textwidth} C{0.30\textwidth}}

 & $I_{\emptyset}(B)$
 & $I_{G}(B)$ 
 & $I_{H}(B)$ \\

 & \hspace{-0.8cm}\includegraphics[width=\linewidth]{figures/univariate_importance_diabetes_risk_Y_features_B.png}
 & \hspace{-0.4cm}\includegraphics[width=\linewidth]{figures/diabetes_risk_I_OBS_Y_B_G.png}
 & \hspace{-0.4cm}\includegraphics[width=\linewidth]{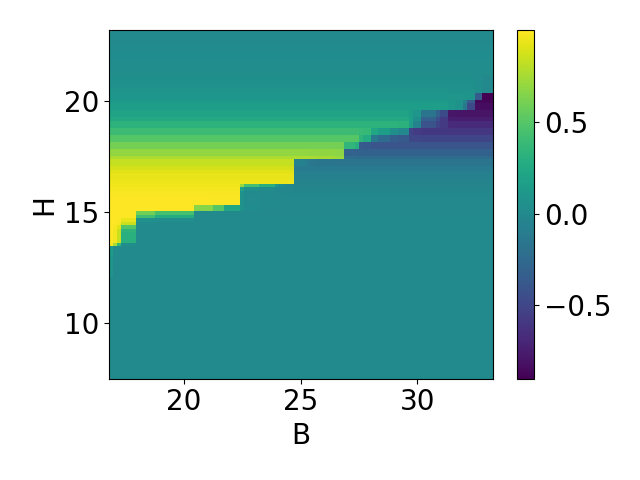}\\ 

 & $I_{B}(G)$
 & $I_\emptyset(G)$ 
 & $I_{H}(G)$ \\

 & \includegraphics[width=\linewidth]{figures/diabetes_risk_I_OBS_Y_G_B.png}
 & \includegraphics[width=\linewidth]{figures/univariate_importance_diabetes_risk_Y_features_G.png}
 & \includegraphics[width=\linewidth]{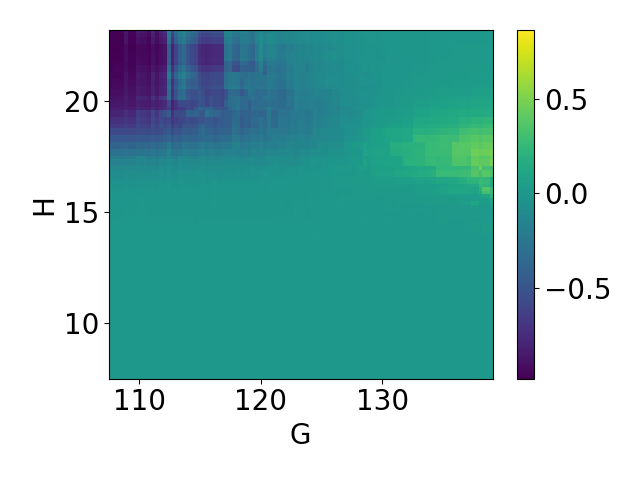}\\ 
 
 & $I_{B}(H)$
 & $I_{G}(H)$ 
 & $I_\emptyset(H)$  \\

 & \includegraphics[width=\linewidth]{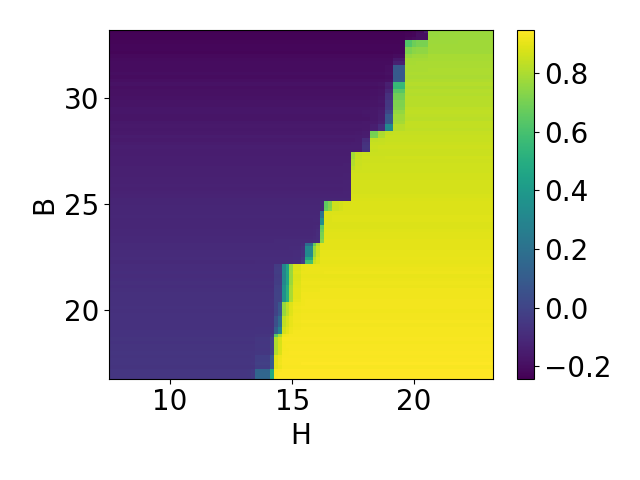}
 & \includegraphics[width=\linewidth]{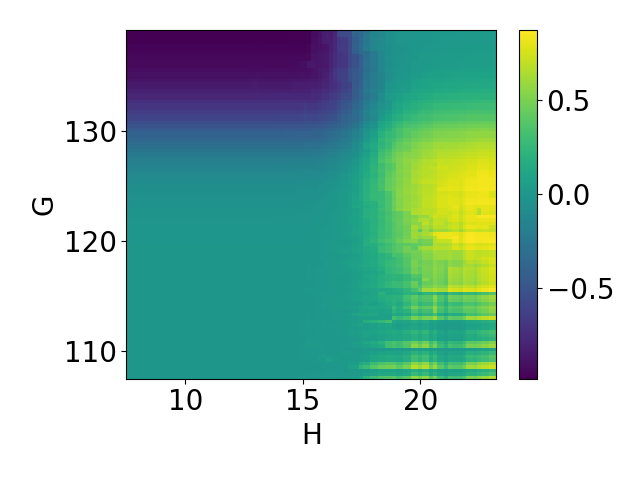}
 & \includegraphics[width=\linewidth]{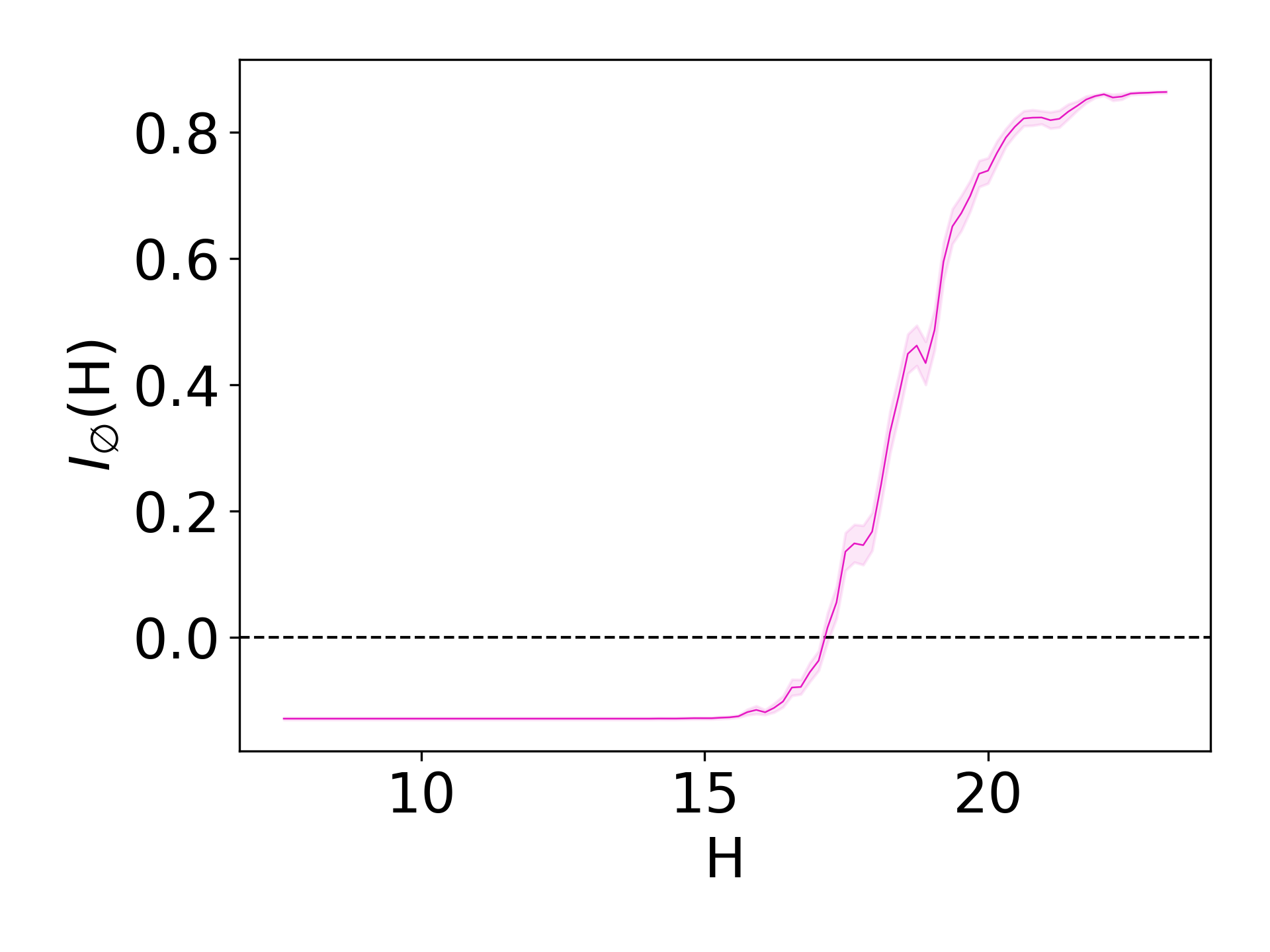}

\end{tabular}
\caption{Plot of all objects with no or univariate context arising in the computation of the conventional Shapley values \eqref{eq:shapley_values} for the SCM \eqref{eq:scm_diabetes_risk}.}
\label{appfig:diabetes_risk_shape_functions_shapley}
\end{figure}

\begin{figure}[p]
\centering
\setlength{\tabcolsep}{6pt}
\renewcommand{\arraystretch}{1.2}

\begin{tabular}{c C{0.30\textwidth} C{0.30\textwidth} C{0.30\textwidth}}

 & $I_{\emptyset}(B)$
 & $I_{\doop(G)}(B)$ 
 & $I_{\doop(H)}(B)$ \\

 & \hspace{-0.8cm}\includegraphics[width=\linewidth]{figures/univariate_importance_diabetes_risk_Y_features_B.png}
 & \hspace{-0.4cm}\includegraphics[width=\linewidth]{figures/diabetes_risk_I_INT_Y_B_G.png}
 & \hspace{-0.4cm}\includegraphics[width=\linewidth]{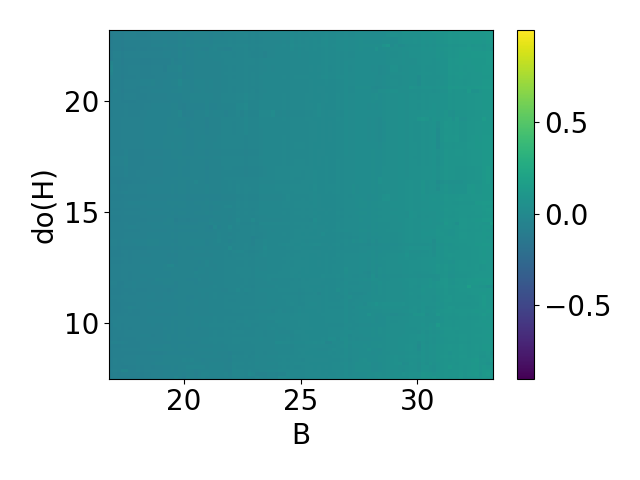}\\ 

 & $I_{\doop(B)}(G)$
 & $I_\emptyset(G)$ 
 & $I_{\doop(H)}(G)$ \\

 & \includegraphics[width=\linewidth]{figures/diabetes_risk_I_INT_Y_G_B.png}
 & \includegraphics[width=\linewidth]{figures/univariate_importance_diabetes_risk_Y_features_G.png}
 & \includegraphics[width=\linewidth]{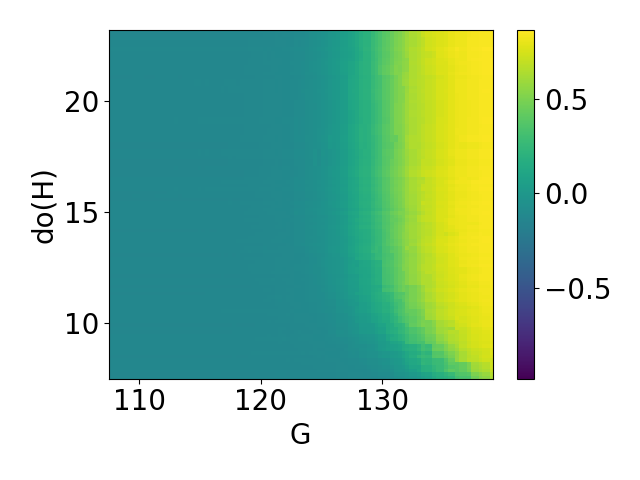}\\ 
 
 & $I_{\doop(B)}(H)$
 & $I_{\doop(G)}(H)$ 
 & $I_\emptyset(H)$  \\

 & \includegraphics[width=\linewidth]{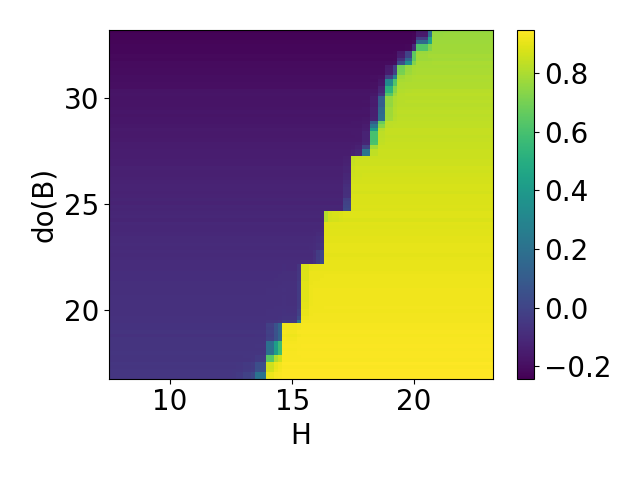}
 & \includegraphics[width=\linewidth]{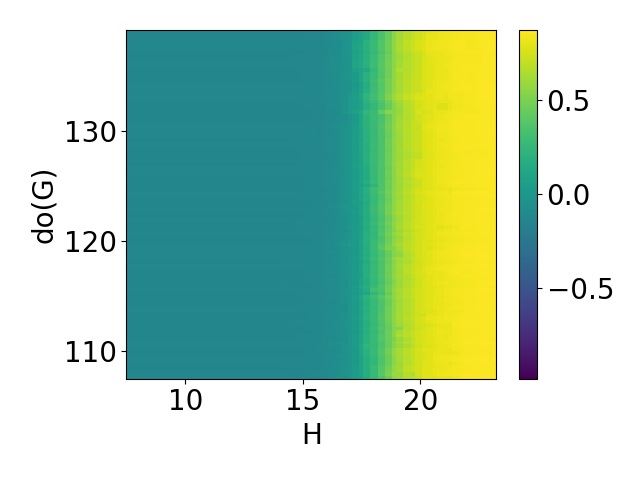}
 & \includegraphics[width=\linewidth]{figures/univariate_importance_diabetes_risk_Y_features_H.png}

\end{tabular}
\caption{Plot of all objects with no or univariate context arising in the computation of the cc-Shapley values \eqref{eq:shapley_values_do} for the SCM \eqref{eq:scm_diabetes_risk}.}
\label{appfig:diabetes_risk_shape_functions_do_shapley}
\end{figure}

The proceeding for \eqref{eq:scm_diabetes_risk} coincides with the simpler example of diabetes detection from Section \ref{app:details_on_diabetes_breakfast} with the only exception that we now have to consider different and more variables in the computation of \eqref{eq:shapley_values} and \eqref{eq:shapley_values_do}. All involved observational $I$ objects with no or univariate context are depicted in Figure \ref{appfig:diabetes_risk_shape_functions_shapley}. The causal analogues for the cc-Shapley values are shown in Figure \ref{appfig:diabetes_risk_shape_functions_do_shapley}.

\paragraph{Derivation of $\phi_{cc}(B)=I_{\emptyset}(B)$.} The sets $\mathcal{S}=\{G\}$, $\mathcal{S}=\{H\}$ and $\mathcal{S}=\{G,H\}$ all have no causal paths to $Y$ or $B$, cf. Figure \ref{fig:causal_graph_diabetes_risk}. Thus, applying Lemma \ref{lem:irrelevant_context} to each of them we obtain $I_{\doop(G)}(B)=I_{\doop(H)}(B)=I_{\doop(\{G,H\})}(B)=I_{\emptyset}(B)$. The claimed identity then follows from the fact that the combinatorial coeffiencents add up to 1 (cf. Section \ref{app:on_the_used_version_of_shapley_values}), in detail
\begin{align}
\phi(B)=\frac{1}{3}I_{\emptyset}(B)+\frac{1}{6}I_{\doop(G)}(B)+\frac{1}{6}I_{\doop(H)}(B)+\frac{1}{3}I_{\doop(\{G,H\})}(B) = I_\emptyset(B) \,.
\end{align}

\newpage
\subsection{Protein dataset from \cite{sachs2005causal}}
\label{app:details_on_sachs}

We used the preprocessed version of the interventional data from \cite{sachs2005causal} available on \cite{bnlearn_sachs05} that bins each variable in the categories $0,1$ and $2$. In addition to the proteins depicted in the causal graph in Figure \ref{fig:causal_graph_sachs} the data contains further information on three proteins \emph{PIP2,PIP3} and \emph{Plcg}. \cite{sachs2005causal} reports their connection to the other proteins in the dataset with a low confidence. We discarded these three proteins for our analysis and only used the ones depicted in Figure \ref{fig:causal_graph_sachs}. In addition, the data contain an indicator that denotes on which of the proteins an experimental intervention happened. For fitting an SCM, we applied the causal graph from Figure \ref{fig:causal_graph_sachs} and fitted an XGBoost classifier for each feature using the parents of each feature as input discarding in each case the data points that were obtained with an intervention on the considered feature (if any).
Subsequently, these models were used as assignment functions to define an SCM on the variables from Figure \ref{fig:causal_graph_sachs}. Since each of the assignment function is a classifier, no assumptions on the underlying noise had to be made. 

Using the SCM as described above, we then computed the Shapley values \ref{eq:shapley_values} and cc-Shapley values \eqref{eq:shapley_values_do}, training, once more, XGBoost classifiers on $10^4$ samples produced via the SCM from above or intervened versions (as described in Algorithm \ref{alg:compute_importance}). As target $Y$, we used the binned concentration of the protein \emph{PKA}. $Y$ can take thus the values $0,1$ and $2$ so that the conditional expectations in \eqref{eq:shapley_values} and \eqref{eq:shapley_values_do} average these numbers using the probabilities produced by the learned classifiers. We evaluated \eqref{eq:shapley_values} and \eqref{eq:shapley_values_do} on $10^4$ (new) datapoints obtained with the learned SCM.

In Figure \ref{fig:shapley_values_sachs_interventional} in Section \ref{sec:experimental_results}, we only showed a subset of the Shapley values and cc-Shapley values. The entire set, including all univariate feature importances, is depicted in Figure \ref{appfig:shapley_values_sachs_interventional}. 

\begin{figure}[h]
    \centering
    \begin{subfigure}[t]{0.49\textwidth}
    \centering
    \includegraphics[width=0.8\linewidth]{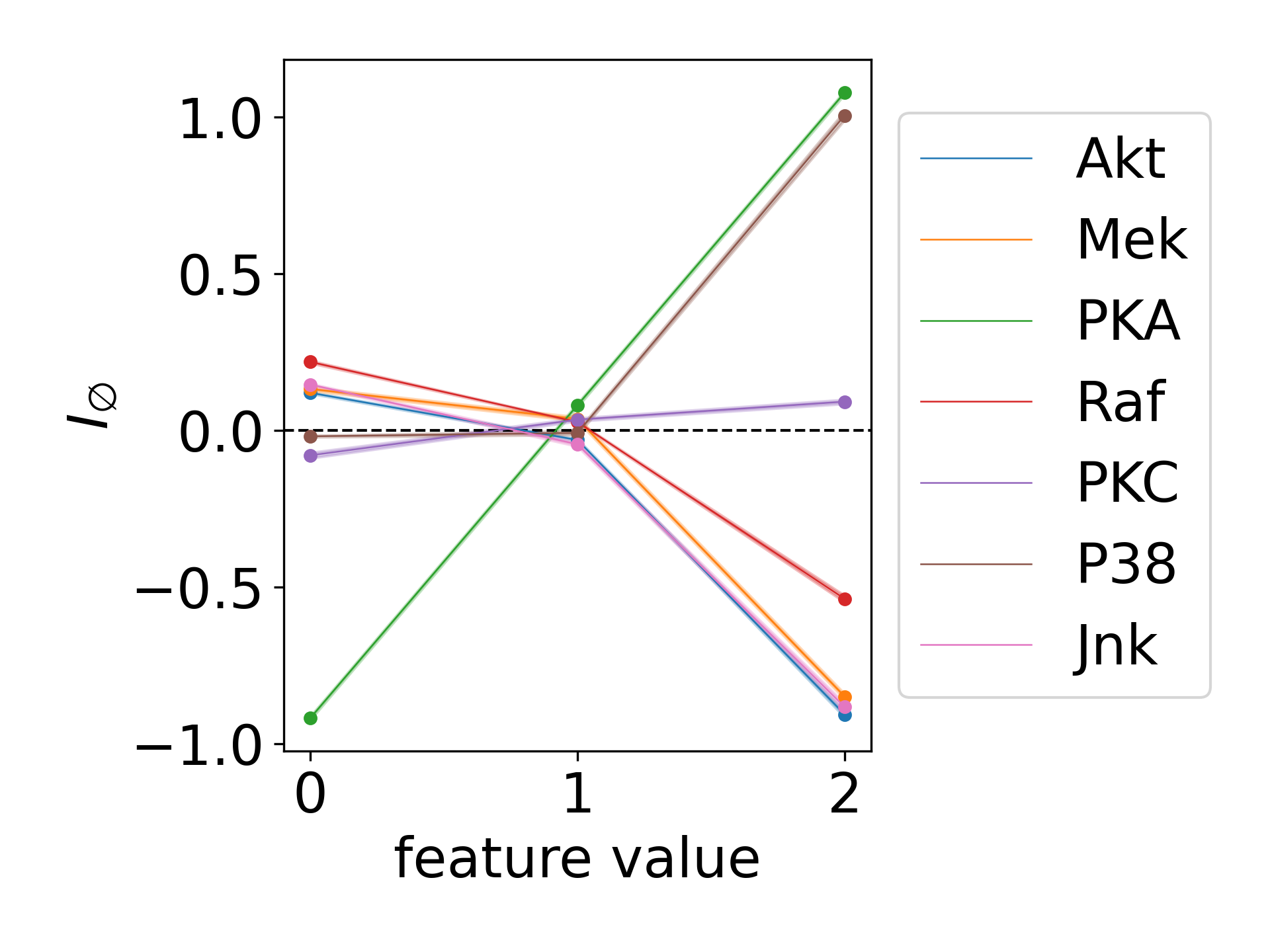}
    \caption{univariate feature importances $I_{\emptyset}$}
    \end{subfigure}
    \begin{subfigure}[t]{0.49\textwidth}
    \includegraphics[width=\linewidth]{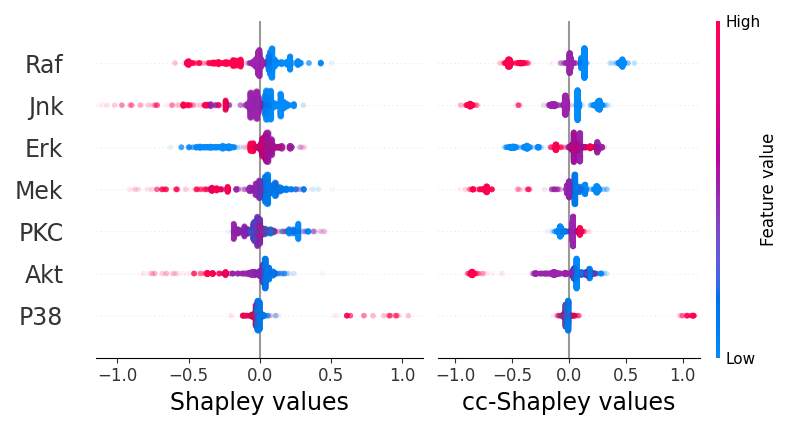}
    \caption{Shapley values and cc-Shapley values}
    \end{subfigure}
    \caption{\emph{Left}: univariate importance of all proteins considered in the experiment with the data from \cite{sachs2005causal}. Standard errors (computed from 5 repetitions) are shown by shaded areas but are very small. \emph{Right}: Shapley values and cc-Shapley values for all considered proteins.}
    \label{appfig:shapley_values_sachs_interventional}
\end{figure}

\end{document}